\documentclass[11pt]{article} 

\usepackage[vmargin=1in,hmargin=1in,centering,letterpaper]{geometry}
\usepackage{times}        
\usepackage{color}
\usepackage{amsmath,amssymb,amsfonts,bm,url,dsfont,amsthm}
\usepackage{natbib}
\usepackage{caption}
\usepackage{algorithm2e}
\usepackage{enumitem}
\usepackage{xparse}
\usepackage{framed}

\newtheorem{definition}{Definition}
\newtheorem{example}{Example}
\newtheorem{problem}{Problem}
\newtheorem{claim}{Claim}
\newtheorem{theorem}{Theorem}
\newtheorem{lemma}{Lemma}
\newtheorem{corollary}{Corollary}
\newtheorem{remark}{Remark}

\newcommand*{\Scale}[2][4]{\scalebox{#1}{$#2$}}%

\DeclareMathOperator*{\argmin}{arg\,min}

\newcommand{\forceindent}[1][1.5em]{\leavevmode{\parindent=#1\indent}}

\newcommand{\rowpad}[2]{\rule{0pt}{#1}\rule[#2]{0pt}{0pt}}

\newif\ifmoveproofs

\DeclareDocumentCommand{\makeproof}{ s m +m  }{
\expandafter\newcommand\csname#2\endcsname{#3}

\ifmoveproofs
\else
\IfBooleanTF{#1}
  {\vspace{-0.4cm}
   \expandafter\ \csname#2\endcsname
  }{\begin{proof}
     \expandafter\ \csname#2\endcsname
     \end{proof}
  }
\fi
}

\DeclareDocumentCommand{\appendixproof}{ s O{Theorem} m m  }{
\ifmoveproofs
\subsection{Proof of #2~\ref{#3}}\label{#4}
\IfBooleanTF{#1}
  {\vspace{-0.4cm}
   \expandafter\ \csname#4\endcsname
  }{\begin{proof}
     \expandafter\ \csname#4\endcsname
     \end{proof}
  }
\else
\fi
}

\newcommand{\appendixproofsection}[1]{
\ifmoveproofs
\section{#1}
\fi
}

\moveproofstrue 

\def\E{\mathbb{E}}
\def\H{\mathbf{H}}
\def\1{\mathbf{1}}
\def\P{\mathbb{P}}
\def\R{\mathbb{R}}

\title{On the Detection of Mixture Distributions with applications to the Most Biased Coin Problem}
\author{\normalsize{Kevin Jamieson} \texttt{kjamieson@eecs.berkeley.edu}\\
\normalsize {Daniel Haas} \texttt{dhaas@eecs.berkeley.edu}\\
\normalsize{Ben Recht} \texttt{brecht@eecs.berkeley.edu}\\
 \normalsize University of California, Berkeley, CA 94720 USA
}

\begin{document}

\maketitle

\begin{abstract}
This paper studies the trade-off between two different kinds of pure exploration: breadth versus depth. 
The most biased coin problem asks how many total coin flips are required to identify a ``heavy'' coin from an infinite bag containing both ``heavy'' coins with mean $\theta_1 \in (0,1)$, and ``light" coins with mean $\theta_0 \in (0,\theta_1)$, where heavy coins are drawn from the bag with probability $\alpha \in (0,1/2)$.
The key difficulty of this problem lies in distinguishing whether the two kinds of coins have very similar means, or whether heavy coins are just extremely rare.
This problem has applications in crowdsourcing, anomaly detection, and radio spectrum search.
\cite{chandrasekaran2014finding} recently introduced a solution to this problem but it required perfect knowledge of $\theta_0,\theta_1,\alpha$. 
In contrast, we derive algorithms that are adaptive to partial or absent knowledge of the problem parameters. 
Moreover, our techniques generalize beyond coins to more general instances of infinitely many armed bandit problems.
We also prove lower bounds that show our algorithm's upper bounds are tight up to $\log$ factors, and on the way characterize the sample complexity of differentiating between a single parametric distribution and a mixture of two such distributions.
As a result, these bounds have surprising implications both for solutions to the most biased coin problem and for anomaly detection when only partial information about the parameters is known.
\end{abstract}
\section{Introduction}
The trade-off between exploration and exploitation has been an ever-present trope in the online learning literature.
In contrast, this paper studies the trade-off between two different kinds of pure exploration: breadth versus depth. 
Consider a magic bag that contains an infinite number of two kinds of biased coins: ``heavy'' coins with mean $\theta_1 \in (0,1)$ and ``light'' coins with mean $\theta_0 \in (0,\theta_1)$. 
When a player picks a coin from the bag, with probability $\alpha$ the coin is ``heavy'' and with probability $(1-\alpha)$ the coin is ``light.'' 
The player can flip any coin she picks from the bag as many times as she wants, and the goal is to identify a heavy coin. 
The key difficulty of this problem lies in distinguishing whether the two kinds of coins have very similar means, or whether heavy coins are just extremely rare.
That is, how does one balance flipping an individual coin many times to better estimate its mean against considering many new coins to maximize the probability of observing a heavy one. 
It turns out that this toy problem is a useful abstraction to characterize the inherent difficulty of real-world problems including automated hiring of crowd workers for data processing tasks, anomaly and intrusion detection, and discovery of vacant frequencies in the radio spectrum.

The most biased coin problem first came to the attention of the authors of this work when it was presented at COLT 2014 \citep{chandrasekaran2014finding}.
In that work, it was shown that if $\alpha$, $\theta_1$, and $\theta_0$ were known then there exists an algorithm based on the sequential probability ratio test (SPRT) that is optimal in that it minimizes the {\em expected} number of total flips to find a ``heavy'' coin whose posterior probability of being heavy is at least $1-\delta$,  and the expected sample complexity of this algorithm was upper-bounded by
\begin{align} \label{sprt_sample_complexity}
 \frac{16}{(\theta_1-\theta_0)^2} \left( \frac{1-\alpha}{\alpha} + \log\left( \frac{ (1-\alpha)(1-\delta) }{\alpha \delta}\right) \right).
\end{align}
However, the practicality of the proposed algorithm is severely limited as it relies critically on knowing $\alpha$, $\theta_1$, and $\theta_0$ exactly. 
In addition, the algorithm requires more than one coin to be outside the bag at a time ruling out some applications.

\cite{malloy2012quickest} addressed some of the shortcomings of \cite{Chandrasekaran_arxiv} (a preprint of \cite{chandrasekaran2014finding}) by considering both an alternative SPRT procedure and a sequential thresholding procedure. Both of these proposed algorithms consider one coin at a time and never return to previous coins. However, the former requires knowledge of all relevant parameters $\alpha,\theta_0,\theta_1$, and the latter requires knowledge of $\alpha,\theta_0$. Moreover, these results are only presented for the asymptotic case where $\delta \rightarrow 0$.

In this work we propose algorithms that are adaptive to partial or even no knowledge of $\alpha,\theta_0,\theta_1$, are guaranteed to return a heavy coin with probability at least $1-\delta$, and support the setting where just one coin is allowed outside the bag at any given time. In addition, we present lower bounds that nearly match the upper bounds shown for our algorithms.

While coins are a useful analogy, all of our lower and upper bounds extend beyond Bernoulli coins to other distributions (e.g. distributions supported on the interval $[0, 1]$),
though we return to the coin analogy throughout for concreteness. 
Indeed, in pursuit of bounds for the coin problem, we derive upper and lower bounds for a related problem, the detection of mixture distributions with applications to anomaly detection.
As a concrete example of that kind of lower bound shown in this work, suppose we observe a sequence of random variables $X_1,\dots,X_n$ and consider the following hypothesis test:
\begin{problem}
\label{hyp_test_normal}
\begin{align*}
\H_0 & : \forall i \ \ X_{1},\dots,X_n \sim \mathcal{N}(\theta,\sigma^2) \quad\text{ for some $\theta \in \R$}, \\
\H_1 &:  \forall i \ \ X_{1},\dots,X_n \sim (1-\alpha) \mathcal{N}(\theta_0,\sigma^2) + \alpha \ \mathcal{N}(\theta_1,\sigma^2) 
\end{align*} 
\end{problem}
We can show that if $\theta_0,\theta_1,\alpha$ are {\em known} and $\theta = \theta_0$, then it suffices to observe just \\$\max\{1/\alpha, \frac{\sigma^2}{\alpha^2(\theta_1-\theta_0)^2}\log(1/\delta)\}$ samples to determine the correct hypothesis with probability at least $1-\delta$.
However, if $\theta_0,\theta_1,\alpha$ are {\em unknown} (and hence we cannot assume a value for $\theta$), we show that whenever $\frac{(\theta_1-\theta_0)^2}{\sigma^2} \leq 1$, at least $\max\left\{1/\alpha,  \left(\frac{\sigma^2}{\alpha(\theta_1-\theta_0)^2} \right)^2\log(1/\delta) \right\}$ samples in expectation are {\em necessary} to determine the correct hypothesis with probability at least $1-\delta$ (see Appendix~\ref{Gaussian_discussion}).
The unknown parameter case has a simple interpretation for anomaly detection with a fixed mixing component $\alpha$ that gets at the key insights of this work:
if the anomalous distribution is well separated from the null distribution, then detecting an anomalous component is only about as hard as observing just one anomalous sample (i.e. $1/\alpha$---no harder than if the parameters were known) since detection is nearly certain between well-separated distributions.
However, when the two distributions are {\em not} well separated then the sample complexity to detect an anomaly scales like the inverse of the KL divergence {\em squared}! 

In this work, we formally prove the above observations as special cases of more general statements about detecting mixtures. Our main contributions are the following:
\begin{itemize}

\item We characterize the difficulty of distinguishing between a single-parameter distribution and a mixture of two such distributions.
When the parameters are known, detecting the presence of a mixture requires a sample complexity that scales as the expected number of samples to differentiate between the two distributions if given samples from each (i.e. the inverse KL divergence).
However, when the distribution parameters are unknown, we prove lower bounds showing that detecting a mixture is quadratically harder if the distributions are not well-separated.
We then show that this bound applies to any algorithm that solves the most biased coin problem by flipping each coin a fixed number of times \citep[as in][]{malloy2012quickest}.


\item We propose and analyze the sample complexity of several algorithms for the most biased coin problem that are adaptive to partial or no knowledge of the distribution parameters, all of which come within log factors of the information-theoretic lower bound (see Table~\ref{tab:upper-bounds}).
These algorithms actually detect any heavy distribution supported on $[0,1]$, not just Bernoulli coins, and solve a particular instance of the infinite armed bandit problem.
We believe both that our algorithms are the first fully adaptive solution to the most biased coin problem, and that the same approach can be reworked to solve more general instances of the infinite-armed bandit problem in the important case when the arm mean distributions are not fully known.


\end{itemize}

\subsection{Motivation and Related Work}
Data labeling for machine learning applications is often performed by humans, and recent work in the crowdsourcing literature accelerates labeling by organizing workers into pools of labelers and paying them to wait for incoming data~\citep{bernstein2011crowds, haas2015clamshell}.
Because workers hired on marketplaces such as  Amazon's Mechanical Turk~\citep{mturk} vary widely in skill, identifying high-quality workers is an important challenge.
If we model each worker's performance (e.g. accuracy or speed) on a set of tasks as drawn from some distribution on $[0,1]$, then selecting a good worker is equivalent to identifying a worker with a high mean by taking as few total samples as possible from all workers. 
Note that we do not observe a worker's inherent skill or mean directly, we must give them tasks from which we estimate it (like repeatedly flipping a biased coin). 
That is, the identification of good workers is well-modeled by the most biased coin problem. 

One can interpret the most biased coin problem as an infinite armed bandit problem where each coin is an arm. In that setting, \cite{berry1997bandit},~\cite{wang2008algorithms} and~\cite{bonald2013two} prove and refine bounds on the expected cumulative regret of the player, whereas~\cite{carpentier2015simple} focus on the pure exploration setting.
All of this work relies on the assumption that the distribution of the means is parametric and known (though \cite{carpentier2015simple} describes a method to estimate the relevant parameters first). Our setting relies on a different parameterization of the means (i.e. $(1-\alpha) \delta_{\theta_0} + \alpha \delta_{\theta_1}$ where $\delta_x$ is a Dirac delta located at $x$), and we focus on settings in which the relevant parameters are unknown. 

Our lower bounds are based on the detection of the presence of a mixture of two parametric distributions versus just a single distribution of the same family.
There has been extensive work in the estimation of mixture distributions \citep{hardt2014sharp,freund1999estimating}.
This literature usually assumes that the mixture coefficient $\alpha$ is bounded away from $0$ and $1$ to ensure that a sufficient amount of samples are observed from each distribution in the mixture. In contrast, we highlight the challenging regime when  $\alpha$ is arbitrarily small, as is the case in statistical anomaly detection~\citep{eskin2000anomaly,thatte2011parametric,agarwal2006detecting}. The current work differs primarily in that we are in an online setting where we choose to keep sampling or stop, and for the coin problem we must decide how many times to flip each coin, not just a stopping time.

\subsection{Preliminaries}
Let $P$ and $Q$ be two probability distributions with a common measurable space. For simplicity, assume $P$ and $Q$ have the same support.

\begin{definition}
Define the {\em KL Divergence} between $P$ and $Q$ as $KL(P,Q) = \int \log\left(\frac{dP}{dQ} \right) dP$.
\end{definition}

\begin{definition}
Define the {\em $\chi^2$ Divergence} between $P$ and $Q$ as $\chi^2(P,Q) = \int \left(\frac{dP}{dQ} -1 \right)^2 dQ = \int \frac{(dP(x)-dQ(x))^2}{dQ(x)} dx$.
\end{definition}
Note that by Jensen's inequality
\begin{align} \label{KL_chiSq}
KL(P,Q) = \E_P\left[  \log\left(\frac{dP}{dQ} \right) \right] \leq \log\left(\E_P\left[  \frac{dP}{dQ} \right] \right)  = \log\left( \chi^2(P,Q) + 1 \right) \leq \chi^2(P,Q).
\end{align}

\begin{example}[Gaussian]
Let $P = \mathcal{N}(\theta_1,\sigma^2)$ and $Q = \mathcal{N}(\theta_0,\sigma^2)$. Then 
\begin{align*}
KL(P,Q)  = \tfrac{(\theta_1-\theta_0)^2}{2\sigma^2} \quad \text{ and } \quad \chi^2(P,Q) = e^{\frac{(\theta_1-\theta_0)^2}{\sigma^2}}-1.
\end{align*}
\end{example}

\begin{example}[Bernoulli]
Let $P = \text{Bernoulli}(\theta_1)$ and $Q = \text{Bernoulli}(\theta_0)$. Then 
\begin{align*}
\textstyle KL(P,Q)  &= \theta_1 \log(\tfrac{\theta_1}{\theta_0}) + (1-\theta_1) \log( \tfrac{1-\theta_1}{1-\theta_0}), \quad  \text{ and } \quad \chi^2(P,Q) = \tfrac{(\theta_1-\theta_0)^2}{\theta_0(1-\theta_0)}.\\
&\textstyle\leq \frac{ (\theta_1 - \theta_0)^2/2}{ \theta_0 (1-\theta_0) - [ (\theta_1-\theta_0)(2\theta_0-1) ]_+}
\end{align*}
\end{example}

\subsection{The Most Biased Coin Problem Statement}
Let $\theta \in \Theta$ index a family of single-parameter probability density functions $g_\theta$ and fix $\theta_0,\theta_1 \in \Theta$, $\alpha \in [0,1/2]$. For any $\theta \in \Theta$ assume that $g_\theta$ is known to the procedure. Consider a sequence of iid Bernoulli random variables $\xi_i \in \{0,1\}$ for $i=1,2,\dots$ where each $\P( \xi_i = 1 ) = 1-\P( \xi_i = 0) = \alpha$. Let $X_{i,j}$ for $j=1,2,\dots$ be a sequence of random variables drawn from $g_{\theta_1}$ if $\xi_i = 1$ and $g_{\theta_0}$ otherwise, and let $\{ \{ X_{i,j} \}_{j=1}^{M_i} \}_{i=1}^N$ represent the sampling history generated by a procedure for some $N \in \mathbb{N}$ and $(M_1,\dots,M_N) \in \mathbb{N}^N$. For any procedure, let $N(\alpha,\theta_0,\theta_1)$ be the random variable denoting the number of distributions each sampled $M_i(\alpha,\theta_0,\theta_1)$ times for all $i$ when the procedure is applied to the problem defined by fixed $(\alpha,\theta_0,\theta_1)$. 
\begin{definition}
We say a procedure is {\em $\delta$-probably correct} if for all $(\alpha,\theta_0,\theta_1)$ it identifies a ``heavy'' distribution with probability at least $1-\delta$.
\end{definition}
For all procedures that are $\delta$-probably correct and follow Algorithm~\ref{alg:most-biased-coin}, our goal is to provide lower and upper bounds on the quantity $\E[ T(\alpha,\theta_0,\theta_1)] = \E[ \sum_{i=1}^{N(\alpha,\theta_0,\theta_1)} M_i(\alpha,\theta_0,\theta_1) ]$ for any $(\alpha, \theta_0, \theta_1)$.
Note that if $g_\theta = \text{Bernoulli}(\theta)$, then $\E[ T(\alpha,\theta_0,\theta_1)]$ is equivalent to the expected number of total coin flips needed to find a most biased coin. To emphasize this, our results are stated generally, then tied to the special case of Bernoulli coins by way of corollaries. All proofs appear in the appendix.

\begin{algorithm}
\begin{framed}
\textbf{Initialize} an empty history ($N = 0, M=\{\}$).\\
\textbf{Repeat} until heavy distribution declared:\\
\forceindent \textbf{Choose} one of
\begin{enumerate}[leftmargin=4em, topsep=0pt, itemsep=-0.1cm]
\item obtain an additional sample from distribution $i = N$ so that $M_i \leftarrow M_i +1$
\item draw a sample from the $(N+1)$st distribution so that $N \leftarrow N+1$, $M_{N} = 1$
\item declare distribution $i = N$ as heavy
\end{enumerate}
\end{framed}
\caption{Sequential procedure for identifying a heavy distribution. Only the last distribution drawn may be sampled or declared heavy, enforcing the rule that only one coin may be outside the bag at a time.}
\label{alg:most-biased-coin}
\end{algorithm}

\section{Lower bounds}

In this section, we derive lower bounds on the sample complexity of valid procedures. 
Section~\ref{subsec:lower_adaptive} provides a lower bound for any \textit{adaptive} procedure that may choose how many times to sample from each distribution independently, and Section~\ref{subsec:lower_fixed} derives bounds for \textit{fixed sample size} procedures that select an $m \geq 1$ and sample from each distribution exactly $m$ times.
The results in Section~\ref{subsubsec:lower_known_static} apply to procedures with full knowledge of $\alpha, \theta_0, \theta_1$, and Section~\ref{subsubsec:lower_unknown_static} demonstrates that without knowledge of these parameters, the sample complexity becomes much higher.

\subsection{Fully adaptive strategies }
\label{subsec:lower_adaptive}

The following theorem, reproduced from \cite{malloy2012quickest}, describes the sample complexity of any $\delta$-probably correct algorithm for the most biased coin identification problem. Note that this lower bound holds for any procedure, regardless of how adaptive it is or if it returns to previously seen distributions to draw additional samples. 

\begin{theorem} \cite[Theorem 2]{malloy2012quickest}\label{lower_adaptive_known}
Fix $\delta \in (0,1)$. Let $T$ be the total number of samples taken of any procedure that is $\delta$-probably correct in identifying a heavy distribution. Then
\begin{align*}
\E[ T ] \geq c_1 \max\left\{ \frac{1-\delta}{\alpha}, \frac{ (1-\delta)}{\alpha KL(g_{\theta_0} | g_{\theta_1} )} \right\}
\end{align*}
whenever $\alpha \leq c_2 \delta$ where $c_1 ,c_2 \in (0,1)$ are absolute constants.
\end{theorem}

The above theorem is directly applicable to the special case where $g_\theta$ is a Bernoulli distribution, implying a lower bound of $\max\left\{ \frac{1-\delta}{\alpha}, \frac{2\min\{\theta_0(1 - \theta_0), \theta_1(1-\theta_1)\}}{\alpha(\theta_1 - \theta_0)^2}\right\}$ on the most biased coin problem. Our upper bounds for adaptive procedures presented later should be compared to this result.

\subsection{The fixed sample size strategy and the detection of mixtures}

The lower bounds of this section are based on two simple observations. The first observation is that identifying that a specific distribution $i \leq N$ is heavy (i.e. $\xi_i = 1$) is at least as hard as detecting that \textit{any} of the distributions up to time $N$ is heavy. Thus, a lower bound on $\E[T(\alpha,\theta_0,\theta_1)]$ for this strictly easier detection problem is also a lower bound for the identification problem. Thus, we've reduced the problem to a sequential hypothesis test of whether all the observed samples all came from a single distribution or from a mixture of two distributions:

\begin{problem}
\label{hyp_test}
\begin{align*}
\H_0 & : \forall i,j  \ \ X_{i,j} \sim g_\theta \quad\text{ for some $\theta \in \widetilde{\Theta} \subseteq \Theta$}, \\
\H_1 &:  \forall i \ \ \xi_i \sim \text{Bernoulli}(\alpha) , \quad \forall i,j  \ \ X_{i,j}    \sim \begin{cases}  g_{\theta_{0}} & \text{ if } \xi_i = 0 \\  g_{\theta_{1}} & \text{ if } \xi_i = 1  \end{cases}
\end{align*} 
\end{problem}
If $\theta_0$ and $\theta_1$ are close to each other, or if $\alpha$ is very small, or both, it can be very difficult to decide between $\H_0$ and $\H_1$ even if $\alpha, \theta_0, \theta_1$ are known a priori.
Note that if $\widetilde{\Theta} = \{\theta_0\}$ and the parameters are known, any lower bound on the problem also bounds the most biased coin problem with known $\alpha, \theta_0, \theta_1$.
In what follows, for any event $A$, let $\P_i(A)$ and $\E_i[A]$ denote probability and expectation of $A$ under hypothesis $\H_i$ for $i \in \{0,1\}$ (the specific value of $\theta$ in $\H_0$ will be clear from context).

The second observation is characterized in the following claim:
\begin{claim}\label{correctness_claim}
Any procedure that is $\delta$-probably correct also satisfies $\P( N(0,\theta_0,\theta_1) < \infty ) \leq \delta$ for all $\theta_0 < \theta_1$.
\end{claim}
\makeproof{CorrectnessClaimProof}{
Suppose there exists a $\delta$-probably correct procedure with $\P(N(0,\theta_0,\theta_1)<\infty) > \delta$. Then there exists a finite $\hat{n} \in \mathbb{N}$ such that $\P(N(0,\theta_0,\theta_1) \leq  \hat{n}) > \delta$. For some $\epsilon \in (0,1)$ to be defined later, define $\hat{\alpha} = \frac{\log(\tfrac{1}{1-\epsilon})}{2 \hat{n}}$ and note that for this $\hat{\alpha}$, $\P( \bigcap_{i=1}^{ \hat{n}} \{ \xi_i = 0\} ) = (1-\hat\alpha)^{\hat{n}} \geq e^{-2 \hat{n} \hat{\alpha} } \geq 1-\epsilon$. Thus, the probability that the procedure terminates with a light distribution under $\alpha = \hat\alpha$ is at least 
\begin{align*}
\P(N(\hat\alpha,\theta_0,\theta_1) \leq \hat{n}, \cap_{i=1}^{ \hat{n}} \{ \xi_i = 0\} ) &= \P(N(\hat\alpha,\theta_0,\theta_1) \leq \hat{n} | \cap_{i=1}^{ \hat{n}} \{ \xi_i = 0\} ) \P(\cap_{i=1}^{ \hat{n}} \{ \xi_i = 0\}) \\
&= \P(N(0,\theta_0,\theta_1) \leq \hat{n} ) \P(\cap_{i=1}^{ \hat{n}} \{ \xi_i = 0\}) > \delta (1-\epsilon).
\end{align*}
Because we can make $\epsilon$ arbitrarily small, the above display implies that the procedure makes a mistake with probability at least $\delta$, but this is a contradiction as the procedure is $\delta$-probably correct. 
}

Claim~\ref{correctness_claim} allows us to restrict our analysis of Problem~\ref{hyp_test} to procedures that in addition to deciding the hypothesis test, satisfy $\P_0( N < \infty ) \leq \delta$.
This property is instrumental in our ability to prove tight bounds on the sample complexity of the procedures.

\label{subsec:lower_fixed}

The fixed sample size strategy fixes an $m \in \mathbb{N}$ prior to starting the game and samples each distribution exactly $m$ times, i.e. $M_i = m$ for all $i \leq N$. To simplify notation let $f_\theta = g_\theta \otimes \dots \otimes g_\theta$ be the $m$-wise product distribution for any $\theta \in \Theta$. Now our problem is more succinctly described as:
\begin{problem}
\label{hyp_test2}
\begin{align*}
\H_0 & : \forall i \ \ X_{i} \sim f_\theta \quad\text{ for some $\theta \in \widetilde{\Theta} \subseteq \Theta$}, \\
\H_1 &:  \forall i \ \ \xi_i \sim \text{Bernoulli}(\alpha) , \quad \forall i  \ \ X_{i}    \sim \begin{cases}  f_{\theta_{0}} & \text{ if } \xi_i = 0 \\  f_{\theta_{1}} & \text{ if } \xi_i = 1  \end{cases}
\end{align*} 
\end{problem}
In the special case where $g_\theta$ is a Bernoulli distribution, $f_\theta$ can be represented by a Binomial distribution with parameters $(m,\theta)$.

\subsubsection{Sample complexity when parameters are known}
\label{subsubsec:lower_known_static}
Theorem~\ref{fixed_known} characterizes the sample complexity of Problem~\ref{hyp_test2} for any valid procedure.
Note that when $\widetilde{\Theta} = \{\theta_0\}$ and $\theta_0$, $\theta_1$, and $\alpha$ are known, then lower bounding the problem also bounds any fixed sample size procedure that solves the most biased coin problem.

 \begin{theorem} \label{fixed_known}
 Fix $\delta \in (0,1)$. Consider the hypothesis test of Problem~\ref{hyp_test2} for any fixed $\theta \in \widetilde{\Theta} \subseteq \Theta$. Let $N$ be the random number of distributions considered before stopping and declaring a hypothesis. If a procedure satisfies $\P_0( N < \infty ) \leq \delta$ and $\P_1( \cup_{i=1}^N \{ \xi_i = 1\} ) \geq 1-\delta$, then 
  \begin{align*}
  \E_1[ N ] &\geq \max\left\{ \frac{1-\delta}{\alpha} , \frac{\log( \tfrac{1}{\delta}) }{ KL(  \P_1 | \P_0 )} \right\} \geq \max\left\{ \frac{1-\delta}{\alpha} , \frac{\log( \tfrac{1}{\delta}) }{ \chi^2(  \P_1 | \P_0 )} \right\} .
  \end{align*}
 In addition, if $\widetilde{\Theta} = \{\theta_0\}$ then
 \begin{align*}
  \E_1[ N ] &\geq \max\left\{ \frac{1-\delta}{\alpha} , \frac{\log( \tfrac{1}{\delta}) }{ \alpha^2 \chi^2(  f_{\theta_1} | f_{\theta_0} )} \right\} .
  \end{align*}
\end{theorem}
\makeproof{FixedKnownProof}{First, let $N$ be the number of distributions considered at the stopping time $T$. Note that $T\geq N$.
By assumption the procedure satisfies $\P_1( N \geq n | \cap_{i=1}^{n-1} \{ \xi_i = 0 \} ) \geq 1-\delta$ for all $n \in \mathbb{N}$. And
\begin{align*}
\P_1( N \geq n ) \geq \P_1( N \geq n,  \cap_{i=1}^{n-1} \{ \xi_i = 0\} ) &=\P_1( N \geq n | \cap_{i=1}^{n-1} \{ \xi_i = 0 \} ) \P_1(  \cap_{i=1}^{n-1} \{ \xi_i = 0\} ) \\
&\geq (1-\delta) (1-\alpha)^{n-1}
\end{align*}
Thus, $\E_1[N] = \sum_{n=1}^\infty \P_1( N \geq n) \geq (1-\delta)  \sum_{n=1}^\infty (1-\alpha)^{n-1} = \frac{1-\delta}{\alpha}$ which results in the first argument of the $\max$. 

Applying Theorem~2.38 of \cite{siegmund2013sequential} we have
\begin{align*}
\E_1[ N ] \chi^2\left( \P_1 | \P_0 \right) \stackrel{\Scale[.6]{\text{Eqn.~\eqref{KL_chiSq}}}}{\geq} \E_1[ N ] KL\left( \P_1 | \P_0 \right) \stackrel{\Scale[.6]{\text{Thm. 2.38}}}{\geq} \log(\tfrac{1}{P_0( N < \infty )}) \stackrel{\Scale[.6]{\text{assumption}}}{\geq} \log( \tfrac{1}{\delta} ),
\end{align*}
which results in the second argument of the max.

If $\widetilde{\Theta} = \{ \theta_0 \}$ then $\chi^2( \P_1 | \P_0 ) = \chi^2( (1-\alpha) f_{\theta_0} + \alpha f_{\theta_1} | f_{\theta_0} )$ and
\begin{align*}
 \chi^2( (1-\alpha) f_{\theta_0} + \alpha f_{\theta_1} | f_{\theta_0} ) = \int \frac{ \left(  (1-\alpha) f_{\theta_0}(x) + \alpha f_{\theta_1}(x) -  f_{\theta_0}(x)  \right)^2 }{  f_{\theta_0}(x) } dx = \alpha^2 \chi^2(  f_{\theta_1} |  f_{\theta_0} ) 
\end{align*}
Thus, $\E_1[ N ] \geq \frac{ \log( \tfrac{1}{\delta} ) }{ \alpha^2 \chi^2(  f_{\theta_1} |  f_{\theta_0} ) }$ which results in the second part of the theorem.
}

The next corollary relates Theorem~\ref{fixed_known} to the special case where distributions are Bernoulli coins and the objective is to find a heavy coin. The second result of the corollary is similar to that of \citet[Theorem 4]{malloy2012quickest} that considers the limit as $\alpha \rightarrow 0$ and assumes $m$ is sufficiently large (specifically, large enough for the Chernoff-Stein lemma to apply). In contrast, our result holds for all finite $\delta, \alpha, m$.

\begin{corollary} \label{known_all_lower}
Fix $\delta \in (0,1)$, $m \in \mathbb{N}$ and consider the class of algorithms that flips each coin exactly $m$ times and outputs a coin $i \leq N$ as its estimate for a heavy coin. If an algorithm in this class is $\delta$-probably correct then
\begin{align*}
\E[N_m] \geq \max\left\{\frac{1-\delta}{\alpha} , \frac{  \log( \tfrac{1}{\delta})  }{ \alpha^2( e^{ m \frac{(\theta_1-\theta_0)^2}{\theta_0(1-\theta_0)}} - 1 )  } \right\} \geq  \frac{  \theta_0(1-\theta_0) \log( \tfrac{1}{\delta})  }{ m \alpha^2  (\theta_1-\theta_0)^2   }  \1_{m \leq  \frac{\theta_0(1-\theta_0)}{2 (\theta_1-\theta_0)^2}}  \ ,
\end{align*}
however, if we pick the best-case $m$:
\begin{align*}
\min_{m \in \mathbb{N}}\E[ m N_m ] \geq  \frac{(1-\delta)\log\left(\frac{\log(1/\delta)}{\alpha} \right)   }{ \alpha }  \frac{\theta_0(1-\theta_0)}{(\theta_1-\theta_0)^2}  .
\end{align*}
\end{corollary}

\makeproof{KnownAllLowerProof}{For $k=0,1$ let $g_{\theta_k}$ be a Bernoulli distribution with parameter $\theta_k$ and let $f_{\theta_k} = g_{\theta_k} \otimes \dots \otimes g_{\theta_k}$ be a product distribution composed of $m$ $g_{\theta_k}$ distributions. Then
\begin{align*}
\chi^2( g_{\theta_1} | g_{\theta_0} ) = \frac{(\theta_1-\theta_0)^2}{\theta_0(1-\theta_0)} \leq e^{ \frac{(\theta_1-\theta_0)^2}{\theta_0(1-\theta_0)}} - 1
\end{align*}
and 
\begin{align*}
\chi^2( f_{\theta_1} | f_{\theta_0} ) = \left( 1 + \chi^2( g_{\theta_1} | g_{\theta_0} ) \right)^m -1 \leq  e^{ m \frac{(\theta_1-\theta_0)^2}{\theta_0(1-\theta_0)}} - 1.
\end{align*}
Moreover, $e^{ m \frac{(\theta_1-\theta_0)^2}{\theta_0(1-\theta_0)}} - 1 \leq m \frac{(\theta_1-\theta_0)^2}{\theta_0(1-\theta_0)}$ whenever $m \leq  \frac{\theta_0(1-\theta_0)}{2 (\theta_1-\theta_0)^2}$ since $e^{x/2}-1 \leq x$ for all $x \in [0,1]$. Applying Theorem~\ref{fixed_known} obtains the first result. The second result follows from loosening the integer constraint on $m$ and minimizing the the lower bound on $\E[N_m]$ multiplied by $m$. To perform the minimization, we note that the function $\max\{\frac{1-\delta}{\alpha} , 2  \log( \tfrac{1}{\delta}) / [\alpha^2( e^{ m \frac{(\theta_1-\theta_0)^2}{\theta_0(1-\theta_0)}} - 1 )] \}$ reaches its minimum at the intersection of the two arguments and solve for $m$ at that point.}

\begin{remark}For all sufficiently small $\frac{(\theta_1-\theta_0)^2}{\theta_0(1-\theta_0)}$, the expected number of flips of the fixed strategy to identify a heavy coin scales like $\Omega( \frac{\theta_0(1-\theta_0)}{\alpha^2 (\theta_1-\theta_0)^2}\log(1/\delta) )$, a factor $1/\alpha$ more than \eqref{sprt_sample_complexity} and the best adaptive algorithms we propose in Section~\ref{sec:upper_bounds} that can identify a heavy coin with just $O( \frac{\log(1/\delta)}{\alpha (\theta_1-\theta_0)^2} )$ total flips in expectation. Indeed, even the lower bound for the best case $m$ is a factor of $\log(1/\alpha)$ from the best upper bounds.
\end{remark}

\subsubsection{Sample complexity when parameters are unknown}
\label{subsubsec:lower_unknown_static}
If $\alpha$, $\theta_0$, and $\theta_1$ are unknown, we cannot test $f_{\theta_0}$ against the mixture $(1-\alpha)f_{\theta_0} + \alpha f_{\theta_1}$. 
Instead, we have the general composite test of \textit{any} individual distribution against \textit{any} mixture, which is at least as hard as the hypothesis test of Problem~\ref{hyp_test2} with $\widetilde{\Theta} = \{\theta\}$ for some particular worst-case setting of $\theta$.
Without any specific form of $f_\theta$, it is difficult to pick a worst case $\theta$ that will produce a tight bound.
Consequently, in this section we appeal to single parameter exponential families (defined formally below) to provide us with a class of distributions in which we can reason about different possible values for $\theta$.
Since exponential families include Bernoulli, Gaussian, exponential, and many other distributions, the following theorem is general enough to be useful in a wide variety of settings.

\begin{theorem} \label{exp_family}
Suppose $f_\theta$ for $\theta \in \Theta \subset \R$ is a single parameter exponential family so that $f_\theta(x) = h(x) \exp( \eta(\theta) x - b(\eta(\theta)) )$ for some scalar functions $h,b,\eta$ where $\eta$ is strictly increasing. If $\E_\theta[X]= \int x f_\theta(x) dx$ then let $M_k(\theta) = \int (x- \E_\theta[X] )^k f_\theta(x) dx$ denote the $k$th centered moment under distribution $f_\theta$. Define 
\begin{align*}
\theta_* &= \eta^{-1}\big((1-\alpha) \eta(\theta_0) + \alpha \eta(\theta_1) \big) \\
\theta_- &= \eta^{-1}\big( \eta(\theta_0) - \alpha( \eta(\theta_1) - \eta(\theta_0) )\big) \\
\theta_+ &= \eta^{-1}\big( \eta(\theta_1) + (1-\alpha)( \eta(\theta_1) - \eta(\theta_0) )\big) 
\end{align*}
and assume there exist finite $\kappa,\gamma$ such that
\begin{align*}
\sup_{y \in [\theta_0,\theta_1]}   b( 2\eta(y)-\eta(\theta_*)) - [2b(\eta(y))-b(\eta(\theta_*))]  \leq \kappa ,\\
\sup_{x \in [ \dot{b}(\eta(\theta_-)),\dot{b}(\eta(\theta_+))]}  \phi_x( \dot{b}^{-1}(x))  \leq \gamma ,
\end{align*}
where  $\phi_x(\eta(\theta)) = f_\theta(x)$. Then 
\begin{align*}
\chi^2( (1-\alpha) f_{\theta_0}(x) + \alpha f_{\theta_1}(x) | f_{\theta_*}(x) )  &\leq c  \left( \frac{1}{2}\alpha(1-\alpha) (\eta(\theta_1)-\eta(\theta_0))^2 \right)^2
\end{align*}
where
\begin{align*}
\hspace{.5in}&\hspace{-.5in}c= e^\kappa \bigg(\sup_{\theta \in [\theta_0,\theta_1]} M_2(\theta)^2 \ \left( 2 + \gamma \left(\dot{b}(\eta(\theta_+))-\dot{b}(\eta(\theta_-))\right)  \right) \\
 &+ 8 M_4(\theta_- ) + 8 M_4( \theta_+ )  + 16\left( \dot{b}( \eta(\theta_+) ) - \dot{b}( \eta(\theta_-)) \right)^4  + \tfrac{2}{5} \gamma \left( \dot{b}(\eta(\theta_+))-\dot{b}(\eta(\theta_-))\right)^5 \bigg).
\end{align*}
Thus, if $\widetilde{\Theta} = \{\theta_*\}$ and $N$ is the stopping time of any procedure that satisfies $\P_0( N < \infty ) \leq \delta$ and $\P_1( \cup_{i=1}^N \{ \xi_i = 1\} ) \geq 1-\delta$, then
 \begin{align*}
  \E_1[ N ] &\geq \max\left\{ \frac{1-\delta}{\alpha} , \frac{\log( \tfrac{1}{\delta}) }{ c  \left( \frac{1}{2}\alpha(1-\alpha) (\eta(\theta_1)-\eta(\theta_0))^2 \right)^2} \right\} .
  \end{align*}

\end{theorem}
\makeproof{ExpFamilyProof}{Define $\phi_x(\eta) = h(x) \exp( \eta x - b(\eta) )$. By the properties of scalar exponential families, note that $b'(\eta)$ and $b''(\eta)\geq 0$ represent the mean and variance of the distribution. We deduce that $b'$ is monotonically increasing. Define $\eta_0 = \eta(\theta_0)$, $\eta_1 = \eta(\theta_1)$, and  $\mu = (1-\alpha) \eta_0 + \alpha \eta_1$. Noting that
\begin{align*}
 \hspace{1in}&\hspace{-1in} \chi^2( (1-\alpha) \phi_x(\eta_0) + \alpha \phi_x(\eta_1) | \phi_x(\mu) ) = \int \phi_x( \mu ) \left(\frac{  (1-\alpha) \phi_x(\eta_0) + \alpha \phi_x(\eta_1) - \phi_x(\mu) }{ \phi_x(\mu) } \right)^2 dx
 \end{align*}
we will use a technique that was used in \cite{pollard2000asymptopia} to approximate the divergence between a single Gaussian distribution and a mixture of them. Essentially, we will take the Taylor series of each $\phi_x( \cdot )$ centered at $\mu$ and bound. We have
\begin{align*}
\phi_x(\eta) &= h(x) \exp( \eta x - b(\eta) ) \\
\phi_x'(\eta) &= (x - b'(\eta) ) \phi_x(\eta) \\
\phi_x''(\eta) &= (- b''(\eta) + (x - b'(\eta) )^2 ) \phi_x(\eta)
\end{align*}
so that
\begin{align*}
\phi_x(y)  = \phi_x(\mu) \left[ 1 +  (x - b'(\mu) )(y-\mu) +  \tfrac{1}{2}(- b''(\mu) + (x - b'(\mu) )^2 ) (y-\mu)^2  \dots \right].
\end{align*}
Noting that  $(\eta_0-\mu)=-\alpha(\eta_1-\eta_0)$, $(\eta_1-\mu)=(1-\alpha)(\eta_1-\eta_0)$, and $(1-\alpha) \alpha^2 + \alpha (1-\alpha)^2  = \alpha(1-\alpha)$, we have 
\begin{align*}
 \hspace{.1in}&\hspace{-.1in} \left|\frac{  (1-\alpha) \phi_x(\eta_0) + \alpha \phi_x(\eta_1) - \phi_x(\mu) }{ \phi_x(\mu) } \right| \\
&= \left| \frac{\phi_x'(\mu)}{\phi_x(\mu)} [ (1-\alpha)(\eta_0-\mu) + \alpha (\eta_1-\mu) ] + \frac{1}{2}\frac{\phi_x''(\mu)}{\phi_x(\mu)} [ (1-\alpha)(\eta_0-\mu)^2 + \alpha (\eta_1-\mu)^2 ] + \dots \right| \\
&= \left|  \frac{1}{2}\frac{\phi_x''(\mu)}{\phi_x(\mu)}\alpha(1-\alpha) (\eta_1-\eta_0)^2 + \dots \right|  \\
&\leq \sup_{z \in [\eta_0,\eta_1]}     \frac{ \left| \phi_x''(z) \right| }{\phi_x(\mu)}   \frac{1}{2}\alpha(1-\alpha) (\eta_1-\eta_0)^2 .
\end{align*}
Thus,
\begin{align*}
 \hspace{1in}&\hspace{-1in} \chi^2( (1-\alpha) \phi_x(\eta_0) + \alpha \phi_x(\eta_1) | \phi_x(\mu) ) = \int \phi_x( \mu ) \left(\frac{  (1-\alpha) \phi_x(\eta_0) + \alpha \phi_x(\eta_1) - \phi_x(\mu) }{ \phi_x(\mu) } \right)^2 dx \\
 &\leq  \left( \frac{1}{2}\alpha(1-\alpha) (\eta_1-\eta_0)^2 \right)^2 \int   \sup_{z \in [\eta_0,\eta_1]}   \frac{\left| \phi_x''(z) \right|^2 }{\phi_x(\mu)^2} \ \phi_x(\mu)  dx .
\end{align*}
By distributing the square and noting that $b''(\eta) \geq 0$, we have
\begin{align*}
  \hspace{.15in}&\hspace{-.15in}  \int   \sup_{z \in [\eta_0,\eta_1]}   \frac{\left| \phi_x''(z) \right|^2 }{\phi_x(\mu)^2} \ \phi_x(\mu)  dx  = \int   \sup_{z \in [\eta_0,\eta_1]}   \left( \frac{ \phi_x(z) }{\phi_x(\mu)} \right)^2 (- b''(z) + (x - b'(z) )^2 )^2 \ \phi_x(\mu)  dx \\
  &\leq  \int   \sup_{z \in [\eta_0,\eta_1]}   \left( \frac{ \phi_x(z) }{\phi_x(\mu)} \right)^2 b''(z)^2  \ \phi_x(\mu)  dx + \int   \sup_{z \in [\eta_0,\eta_1]}   \left( \frac{ \phi_x(z) }{\phi_x(\mu)} \right)^2 (x - b'(z) )^4  \ \phi_x(\mu)  dx \\
    &\leq \sup_{y \in [\eta_0, \eta_1]}  b''(y)^2  \int   \sup_{z \in [\eta_0,\eta_1]}   \left( \frac{ \phi_x(z) }{\phi_x(\mu)} \right)^2 \ \phi_x(\mu)  dx + \int   \sup_{z \in [\eta_0,\eta_1]}   \left( \frac{ \phi_x(z) }{\phi_x(\mu)} \right)^2 (x - b'(z) )^4  \ \phi_x(\mu)  dx .
\end{align*}
The remainder of the proof bounds the integrals. Define $\eta_-=2\eta_0-\mu=\eta(\theta_-)$ and $\eta_-=2\eta_1-\mu=\eta(\theta_+)$. Observe that 
\begin{align*}
 &\sup_{z \in [\eta_0,\eta_1]}   \left( \frac{ \phi_x(z) }{\phi_x(\mu)} \right)^2 \ \phi_x(\mu)   \\ 
 &=  \sup_{z \in [\eta_0,\eta_1]} h(x) \exp\big( (2 z- \mu ) x - (2b(z)-b(\mu)) \big) \\
&=  \sup_{z \in [\eta_0,\eta_1]} h(x) \exp\big( (2z-\mu ) x - b( 2z-\mu) \big) \exp\big( b( 2z-\mu) - (2b(z)-b(\mu)) \big) \\
&\leq e^\kappa \sup_{z \in [\eta_0,\eta_1]} h(x) \exp\big( (2z-\mu ) x - b( 2z-\mu) \big)  \\
&= e^\kappa \sup_{z \in [ 2 \eta_0 -\mu, 2\eta_1-\mu]} h(x) \exp\big( z x - b( z ) \big)  \\
&= e^\kappa \sup_{z \in [ \eta_-, \eta_+]} h(x) \exp\big( z x - b( z ) \big)  \\
&\leq e^\kappa \left( \phi_x(\eta_-) + \phi_x( \eta_+ ) + \phi_x( \dot{b}^{-1}(x)) \1_{x \in [ \dot{b}(\eta_-),\dot{b}(\eta_+)]} \right) \\
&\leq e^\kappa \left( \phi_x(\eta_-) + \phi_x( \eta_+ ) + \gamma \1_{x \in [ \dot{b}(\eta_-),\dot{b}(\eta_+)]} \right) 
\end{align*} 
where the second inequality follows by observing that the maximum of the function $\phi_x(z)$ will occur either at an endpoint of the interval $z\in[\eta(\theta_-), \eta(\theta_+)]$ or at the point where $\frac{\partial}{\partial z} g(z) = 0$ (if that point occurs inside the interval), and loosely bounding the maximum by simply adding the function values at all three points.

Consequently,
\begin{align*}
  &\sup_{y \in [\eta_0, \eta_1]}  b''(y)^2   \int   \sup_{z \in [\eta_0,\eta_1]}   \left( \frac{ \phi_x(z) }{\phi_x(\mu)} \right)^2 \ \phi_x(\mu)  dx   &\leq \sup_{\theta \in [\theta_0,\theta_1]} M_2(\theta)^2  e^\kappa \left( 2 + \gamma (\dot{b}(\eta_+)-\dot{b}(\eta_-))  \right).
 \end{align*}
By Jensen's inequality, $(a+b)^4 = 16 ( \tfrac{1}{2} a + \tfrac{1}{2} b )^4 \leq 8 ( a^4 + b^4)$, so 
\begin{align*}
 \int \sup_{z \in [\eta_0,\eta_1]} \phi_x(\eta_-) (x - \dot{b}(z) )^4 dx &=  \int \sup_{z \in [\eta_0,\eta_1]} \phi_x(\eta_-) (x- \dot{b}(\eta_-)  + \dot{b}(\eta_-)  - \dot{b}(z))^4 dx \\
&\leq \int 8 \phi_x(\eta_-)  [  (x - \dot{b}(\eta_-)  )^4 +   \sup_{z \in [\eta_0,\eta_1]} ( \dot{b}(\eta_-)  - \dot{b}(z) )^4 ] dx \\
&\leq \int 8  \phi_x(\eta_-)  [  (x - \dot{b}(\eta_-)  )^4  +    (\dot{b}(\eta_-) - \dot{b}(\eta_1) )^4 ] dx \\
&=  8 [  M_4( \theta_- ) -   (\dot{b}(\eta_-) - \dot{b}(\eta_1) )^4  ] .
\end{align*} 
Repeating an analogous series of steps for $\eta_+$, we have 
\begin{align*}
 & \int   \sup_{z \in [\eta_0,\eta_1]}   \left( \frac{ \phi_x(z) }{\phi_x(\mu)} \right)^2 (x - b'(z) )^4  \ \phi_x(\mu)  dx \\
 &\leq e^\kappa \int  \left( \phi_x(\eta_-) + \phi_x( \eta_+ ) + \gamma \1_{x \in [ \dot{b}(\eta_-),\dot{b}(\eta_+)]} \right)   \sup_{z \in [\eta_0,\eta_1]}   (x - \dot{b}(z) )^4    dx \\
 &\leq  e^\kappa \left( 8 M_4( \theta_- ) + 8(\dot{b}(\eta_1) - \dot{b}(\eta_-)  )^4 + 8 M_4(  \theta_+ )  + 8( \dot{b}(\eta_+) - \dot{b}(\eta_0)  )^4  + \tfrac{2}{5} \gamma ( \dot{b}(\eta_+)-\dot{b}(\eta_-))^5 \right)\\
 &\leq  e^\kappa \left( 8 M_4( \theta_- ) + 8 M_4( \theta_+ )  + 16( \dot{b}(\eta_+) - \dot{b}(\eta_-) )^4  + \tfrac{2}{5} \gamma ( \dot{b}(\eta_+)-\dot{b}(\eta_-))^5 \right).
 \end{align*}
 The final result holds by Theorem~\ref{fixed_known}.}

Theorem~\ref{exp_family} is difficult to interpret, so the following remark and corollary consider the special cases of Gaussian mixture model detection and the most biased coin problem, respectively.

\begin{remark}
Recall that when $\alpha,\theta_0,\theta_1$ are unknown, any procedure does not know how to choose $\widetilde{\Theta}$ in Problem~\ref{hyp_test2} and consequently it cannot rule out $\theta=\theta_*$ for $\H_0$ where $\theta_*$ is defined in Theorem~\ref{exp_family}.  If $f_\theta = \mathcal{N}(\theta,\sigma^2)$ for known $\sigma$, then whenever $\frac{(\theta_1-\theta_0)^2}{\sigma^2} \leq 1$ the constant $c$ in Theorem~\ref{exp_family} is an absolute constant and consequently, $\E_1[N] = \Omega\left( \left( \frac{\sigma^2}{ \alpha (\theta_1-\theta_0)^2 } \right)^2 \log(1/\delta) \right)$. Conversely, when $\alpha,\theta_0,\theta_1$ are known, then we simply need to determine whether samples came from $\mathcal{N}(\theta_0,\sigma^2)$ or $(1-\alpha) \mathcal{N}(\theta_0,\sigma^2) + \alpha \mathcal{N}(\theta_1,\sigma^2)$, and we show that it is sufficient to take just $O\left( \frac{\sigma^2}{ \alpha^2 (\theta_1-\theta_0)^2 } \log(1/\delta) \right)$ samples (see Appendix~\ref{Gaussian_discussion}).
\end{remark}

\begin{corollary} \label{fixed_unknown}
Fix $\delta \in [0, 1], m \in \mathbb{N}$ and consider the class of algorithms that flips each coin exactly $m$ times. Assume $\theta_0,\theta_1$ are bounded sufficiently far from $\{0, 1\}$ such that $2(\theta_1-\theta_0) \leq \min\{ \theta_0(1-\theta_0) , \theta_1(1-\theta_1) \}$. If an algorithm in this class is $\delta$-probably correct then
\begin{align*}
\E[N] \geq  \frac{ c' \min\{ \frac{1}{m} , \theta_*(1-\theta_*) \}  }{m \left(  \alpha(1-\alpha) \frac{(\theta_1-\theta_0)^2}{\theta_*(1-\theta_*)} \right)^2 }\log(\tfrac{1}{\delta})   \quad \text{whenever}\quad  m \leq \frac{\theta_*(1-\theta_*)}{(\theta_1-\theta_0)^2}.
\end{align*} 
where $c'$ is an absolute constant and $\theta_* = \eta^{-1}\left( (1-\alpha) \eta(\theta_0) + \alpha \eta(\theta_1) \right) \in [\theta_0,\theta_1]$.
\end{corollary}
\makeproof{FixedUnknownProof}{A binomial distribution for fixed $m$ is an exponential family $f_\theta(x) = h(x) \exp( \eta(\theta) x - b(\eta(\theta)) )$ with $h(x) = \binom{m}{x}$, $\eta(\theta) = \log( \tfrac{\theta}{1-\theta} )$, and $b(\tau) = m \log(1+e^\tau)$. Note that $\eta$ is monotonically increasing, $b$ is $m$-Lipschitz, and $\dot{b}(\tau) = m(1 + e^{-\tau})^{-1}$ so that $\dot{b}(\eta(\theta)) = m\theta$. 

\noindent\textbf{Step 1: Relating $\theta_+,\theta_-$ to $\theta_1,\theta_0$}\\
We will make repeated use of the fact that if $f$ is convex then $f(y) \geq f(x) +  f'(x)^T (y-x)$. Since $\frac{x}{1-x}$ and $\frac{1-x}{x}$ are both convex, we have
\begin{align*}
\frac{y}{1-y} \geq \frac{x}{1-x} + \frac{y-x}{(1-x)^2} \quad \text{ and } \quad \frac{1-y}{y} \geq \frac{1-x}{x} - \frac{y-x}{x^2}
\end{align*}
for all $x,y \in [0,1]$.

To begin, note $\eta^{-1}(\nu) = (1+e^{-\nu})^{-1}$ so that for any $\theta$ we have $\theta(1-\theta) = \eta^{-1}( \eta(\theta) ) ( 1- \eta^{-1}(\eta(\theta))) = \frac{e^{-\eta(\theta)}}{(1+e^{-\eta(\theta)})^2}$. Observe that
\begin{align*}
\frac{1}{4} e^{-|\eta(\theta)|} \leq  \frac{e^{-\eta(\theta)}}{(1+e^{-\eta(\theta)})^2} \leq e^{-|\eta(\theta)|}
\end{align*}
and recalling that $\theta_* = \eta^{-1}( (1-\alpha)\theta_0 + \alpha \theta_1 ) \in [ \theta_0,\theta_1]$ we have
\begin{align*}
\theta_+(1-\theta_+) &\geq \frac{1}{4} e^{-|\eta(\theta_+)|} = \frac{1}{4} e^{-|2\eta(\theta_1)-\eta(\theta_*)|} \\
&\hspace{-.5in}= \tfrac{1}{4}\1_{\theta_+ \leq 1/2} \left(\frac{\theta_1}{1-\theta_1} \right)^2 \left(\frac{1-\theta_*}{\theta_*} \right) + \tfrac{1}{4} \1_{\theta_+ > 1/2} \left(\frac{1-\theta_1}{\theta_1} \right)^2 \left(\frac{\theta_*}{1-\theta_*} \right) \\
&\hspace{-.5in}\geq \tfrac{1}{4}\1_{\theta_+ \leq 1/2} \left(\frac{\theta_1}{1-\theta_1} \right)^2 \left(\frac{1-\theta_1}{\theta_1} \right) + \tfrac{1}{4} \1_{\theta_+ > 1/2} \left(\frac{1-\theta_1}{\theta_1} \right)^2 \left(\frac{\theta_0}{1-\theta_0} \right) \\
&\hspace{-.5in}\geq \tfrac{1}{4}\1_{\theta_+ \leq 1/2} \left(\frac{\theta_1}{1-\theta_1} \right) + \tfrac{1}{4} \1_{\theta_+ > 1/2} \left(\frac{1-\theta_1}{\theta_1} \right)^2 \left( \frac{\theta_1}{1-\theta_1} - \frac{\theta_1-\theta_0}{(1-\theta_1)^2} \right) \\ 
&\hspace{-.5in}\geq \tfrac{1}{4}\1_{\theta_+ \leq 1/2} \left(\frac{\theta_1}{1-\theta_1} \right) + \tfrac{1}{8} \1_{\theta_+ > 1/2} \left(\frac{1-\theta_1}{\theta_1} \right) \geq \frac{1}{8} \theta_1(1-\theta_1)
\end{align*}
where the last line follows from the assumption that $\theta_1(1-\theta_1) \geq 2(\theta_1-\theta_0)$. Analogously,
\begin{align*}
\theta_-(1-\theta_-) &\geq \frac{1}{4} e^{-|\eta(\theta_-)|} = \frac{1}{4} e^{-|2\eta(\theta_0)-\eta(\theta_*)|} \\
&\hspace{-.5in}= \tfrac{1}{4}\1_{\theta_- \leq 1/2} \left(\frac{\theta_0}{1-\theta_0} \right)^2 \left(\frac{1-\theta_*}{\theta_*} \right) + \tfrac{1}{4} \1_{\theta_- > 1/2} \left(\frac{1-\theta_0}{\theta_0} \right)^2 \left(\frac{\theta_*}{1-\theta_*} \right) \\
&\hspace{-.5in}\geq \tfrac{1}{4}\1_{\theta_- \leq 1/2} \left(\frac{\theta_0}{1-\theta_0} \right)^2 \left(\frac{1-\theta_1}{\theta_1} \right) + \tfrac{1}{4} \1_{\theta_- > 1/2} \left(\frac{1-\theta_0}{\theta_0} \right)^2 \left(\frac{\theta_0}{1-\theta_0} \right) \\
&\hspace{-.5in}\geq \tfrac{1}{4}\1_{\theta_- \leq 1/2} \left(\frac{\theta_0}{1-\theta_0} \right)^2 \left(\frac{1-\theta_0}{\theta_0} - \frac{\theta_1-\theta_0}{\theta_0^2} \right) + \tfrac{1}{4} \1_{\theta_- > 1/2} \left(\frac{1-\theta_0}{\theta_0} \right) \\
&\hspace{-.5in}\geq \tfrac{1}{8}\1_{\theta_- \leq 1/2} \left(\frac{\theta_0}{1-\theta_0} \right) + \tfrac{1}{4} \1_{\theta_- > 1/2} \left(\frac{1-\theta_0}{\theta_0} \right) \geq \frac{1}{8} \theta_0(1-\theta_0)
\end{align*}
where the last line follows from the assumption that $\theta_0(1-\theta_0) \geq 2(\theta_1-\theta_0)$. We conclude that
\begin{align} \label{theta_pm_inf}
\inf_{\theta \in [\theta_-,\theta_+]} \theta(1-\theta) \geq \frac{1}{8} \inf_{\theta \in [\theta_0,\theta_1]} \theta(1-\theta).
\end{align}
Conversely,
\begin{align*}
\sup_{\theta \in [\theta_-,\theta_+]} \theta(1-\theta) \leq \1_{1/2 \in [\theta_-,\theta_+]} \frac{1}{4} +  \theta_+(1-\theta_+) \1_{\theta_+ \leq 1/2} +  \theta_-(1-\theta_-) \1_{\theta_- > 1/2}.
\end{align*}
We consider these three cases in turn. If $\theta_+ \leq 1/2$:
\begin{align*}
\theta_+&(1-\theta_+) \leq e^{-|\eta(\theta_+)|} = e^{-|2\eta(\theta_1)-\eta(\theta_*)|} \\
&= \left(\frac{\theta_1}{1-\theta_1} \right)^2 \left(\frac{1-\theta_*}{\theta_*} \right) \leq \left(\frac{\theta_1}{1-\theta_1} \right)^2 \left(\frac{1-\theta_0}{\theta_0} \right) \leq \left(\frac{\theta_1}{1-\theta_1} \right)^2 \left(\frac{1-\theta_1}{\theta_1} + \frac{\theta_1-\theta_0}{\theta_0^2} \right)  \\
&= \left(\frac{\theta_1}{1-\theta_1} \right)\left(1 + \frac{\theta_1(\theta_1-\theta_0)}{(1-\theta_1)\theta_0^2} \right) \leq \left(\frac{\theta_1}{1-\theta_1} \right)\left(1 + \frac{\theta_1(1-\theta_0)}{2 (1-\theta_1)\theta_0} \right) \\
&= \left(\frac{\theta_1}{1-\theta_1} \right)\left(1 + \frac{\theta_0(1-\theta_0)+(\theta_1-\theta_0)(1-\theta_0)}{2 (1-\theta_1)\theta_0} \right) \\
&\leq \left(\frac{\theta_1}{1-\theta_1} \right)\left(1 + \frac{\theta_0(1-\theta_0)+\theta_0(1-\theta_0)^2/2}{2 (1-\theta_1)\theta_0} \right) \leq \frac{5}{2} \left(\frac{\theta_1}{1-\theta_1} \right) \leq 10 \theta_1(1-\theta_1)
\end{align*}
using the convexity of $\frac{1-x}{x}$, the assumption that $2(\theta_1-\theta_0) \leq \theta_0(1-\theta_0)$, that $\theta_1 \leq \theta_+ \leq 1/2$, and that $1 - \theta_0 \leq 1$. If $\theta_- > 1/2$:
\begin{align*}
\theta_-&(1-\theta_-) \leq e^{-|\eta(\theta_-)|} = e^{-|2\eta(\theta_0)-\eta(\theta_*)|} \\
&= \left(\frac{1-\theta_0}{\theta_0} \right)^2 \left(\frac{\theta_*}{1-\theta_*} \right) \leq \left(\frac{1-\theta_0}{\theta_0} \right)^2 \left(\frac{\theta_1}{1-\theta_1} \right) \leq \left(\frac{1-\theta_0}{\theta_0} \right)^2 \left(\frac{\theta_0}{1-\theta_0} + \frac{\theta_1-\theta_0}{(1-\theta_1)^2} \right) \\
&\leq \left(\frac{1-\theta_0}{\theta_0} \right) \left(1 +  \frac{(1-\theta_0)(\theta_1-\theta_0)}{\theta_0(1-\theta_1)^2} \right) \leq \left(\frac{1-\theta_0}{\theta_0} \right) \left(1 +  \frac{(1-\theta_0) \theta_1/2}{\theta_0(1-\theta_1)} \right) \\
&= \left(\frac{1-\theta_0}{\theta_0} \right) \left(1 +  \frac{(1-\theta_1) \theta_1 + (\theta_1-\theta_0)\theta_1}{2\theta_0(1-\theta_1)} \right) \\
&\leq \left(\frac{1-\theta_0}{\theta_0} \right) \left(1 +  \frac{(1-\theta_1) \theta_1 + (1-\theta_1)\theta_1^2 /2 }{2\theta_0(1-\theta_1)} \right) \leq \frac{5}{2} \left(\frac{1-\theta_0}{\theta_0} \right) \leq 10 \theta_0 (1-\theta_0)
\end{align*}
using the same methods as above. From these two cases, we can conclude that if $1/2 \notin [\theta_-, \theta_+]$,
\begin{align}\label{theta_pm_sup}
\sup_{\theta \in [\theta_-,\theta_+]} \theta(1-\theta) \leq 10 \sup_{\theta \in [\theta_0,\theta_1]} \theta(1-\theta) .
\end{align}

The remaining case, when $1/2 \in [\theta_-,\theta_+]$, also satisfies~\eqref{theta_pm_sup}, which we now demonstrate. When $\theta_+ = 1/2$ we have $1/4 = \theta_+(1-\theta_+) \leq 10 \theta_1(1-\theta_1)$ so that $\theta_1(1-\theta_1) \geq 1/40$. Because $\theta_1$ is monotonically increasing in $\theta_+$ and $\sup_{\theta \in [\theta_-,\theta_+]} \theta(1-\theta) \leq 1/4$ we conclude that \eqref{theta_pm_sup} holds whenever $\theta_1 \leq 1/2$. A similar argument follows for all $\theta_0\geq 1/2$. Finally, if $1/2 \in [\theta_0,\theta_1]$, it must be true that $\sup_{\theta \in [\theta_-, \theta_+]} \theta(1 - \theta) \leq \sup_{\theta \in [\theta_0, \theta_1]} \theta(1 - \theta)$ because $\theta_- \leq \theta_0 \leq \tfrac{1}{2} \leq \theta_1 \leq \theta_+$ and the function $\theta(1-\theta)$ is concave taking its maximum at $\tfrac{1}{2}$. Thus,~\eqref{theta_pm_sup} holds for all $\theta_-, \theta_+$.

We now turn our attention to bounding $\theta_+-\theta_-$.
Let $g(y)=\eta^{-1}(y)$ then $g(y) = (1+e^{-y})^{-1}$ and $\dot{g}(y) = e^{-y}(1+e^{-y})^{-2}$. Observing that $\dot{g}(\eta(\theta)) = \theta(1-\theta)$ we have by Taylor's remainder theorem
\begin{align*}
\theta_+&-\theta_- = \eta^{-1}( \eta(\theta_+) ) - \eta^{-1}(\eta(\theta_-)) \leq \left( \eta(\theta_+)-\eta(\theta_-)\right) \sup_{y \in [\eta(\theta_-),\eta(\theta_+)]}e^{-y}(1+e^{-y})^{-2}  \\
&= \left( \eta(\theta_+)-\eta(\theta_-)\right) \sup_{\theta \in [\theta_-,\theta_+]} \theta(1-\theta) = 2 \left( \eta(\theta_1)-\eta(\theta_0)\right) \sup_{\theta \in [\theta_-,\theta_+]} \theta(1-\theta) \\
&\leq20 \left( \eta(\theta_1)-\eta(\theta_0)\right) \sup_{\theta \in [\theta_0,\theta_1]} \theta(1-\theta) .
\end{align*}
Since $\eta(\theta) = \log(\tfrac{\theta}{ 1-\theta})$ and $\eta'(\theta) = \frac{1}{\theta} + \frac{1}{1-\theta} = \frac{1}{\theta(1-\theta)}$, we have
\begin{align*}
\theta_+-\theta_- \leq 20 \left( \eta(\theta_1)-\eta(\theta_0)\right) \sup_{\theta \in [\theta_0,\theta_1]} \theta(1-\theta) \leq 20\left(\theta_1-\theta_0\right)\frac{\sup_{\theta \in [\theta_0,\theta_1]} \theta(1-\theta)}{\inf_{\theta \in [\theta_0,\theta_1]} \theta(1-\theta)} .
\end{align*}
If $\theta_1(1-\theta_1) \geq \theta_0(1-\theta_0)$:
\begin{align*}
\frac{\theta_1(1-\theta_1)}{\theta_0(1-\theta_0)} &= \frac{\theta_0(1-\theta_1) + (\theta_1-\theta_0)(1-\theta_1)}{\theta_0(1-\theta_0)} \\
&\leq \frac{\theta_0(1-\theta_1) + \theta_0(1-\theta_0)(1-\theta_1)/2}{\theta_0(1-\theta_0)} \leq 1 + (1-\theta_1)/2 \leq 3/2 ,
\end{align*}
else if $\theta_0(1-\theta_0) \geq \theta_1(1-\theta_1)$
\begin{align*}
\frac{\theta_0(1-\theta_0)}{\theta_1(1-\theta_1)} &= \frac{\theta_0(1-\theta_1) + \theta_0(\theta_1-\theta_0)}{\theta_1(1-\theta_1)} \\
&\leq \frac{\theta_0(1-\theta_1) + \theta_0 \theta_1(1-\theta_1)/2}{\theta_1(1-\theta_1)} \leq 1+\theta_0/2 \leq 3/2.
\end{align*}
Finally, if $1/2\in[\theta_0,\theta_1]$ then $\sup_{\theta \in [\theta_0,\theta_1]} \theta(1-\theta)=1/4$ taking its maximum at $1/4$. To maximize the ratio of the $\sup$ to the $\inf$, it suffices to just consider the case when $\theta_0 = 1/2$ or $\theta_1=1/2$. Thus, the above two bounds suffice for this case and we observe that
\begin{align} \label{theta_star_sup_inf}
\frac{\sup_{\theta \in [\theta_0,\theta_1]} \theta(1-\theta)}{\inf_{\theta \in [\theta_0,\theta_1]} \theta(1-\theta)} \leq 3/2.
\end{align}
Thus, putting the pieces together, we conclude that 
\begin{align} \label{theta_pm_diff}
\theta_+-\theta_- \leq 30 (\theta_1-\theta_0).
\end{align}

\noindent{\textbf{Step 2: Bounding $\gamma,\kappa,c$}}\\
In what follows, define $\theta_h = \arg\sup_{\theta \in [\theta_0,\theta_1]} \theta(1-\theta)$ and $\theta_l = \arg\inf_{\theta \in [\theta_0,\theta_1]} \theta(1-\theta)$.
We now continue to bound the terms of the theorem. Note
\begin{align*}
\hspace{.5in}&\hspace{-.5in} \sup_{x \in [ \dot{b}(\eta(\theta_-)),\dot{b}(\eta(\theta_+))]}  \phi_x( \dot{b}^{-1}(x) ) = \sup_{x \in[m\theta_-,m\theta_+]}   \phi_x( \eta(x/m) )\\
&\leq \sup_{x \in[m\theta_-,m\theta_+]} \sup_{y \in [0,1]} \phi_x( \eta(y)) \\
&=\sup_{x \in[m\theta_-,m\theta_+]} \sup_{y \in [0,1]} \frac{\Gamma(m+1)}{\Gamma(m-x+1) \Gamma(x+1)} y^{x} (1-y)^{m-x}\\
&= \sup_{\theta \in[\theta_-,\theta_+]} \sup_{y \in [0,1]} \frac{\Gamma(m+1)}{\Gamma(m(1-\theta)+1) \Gamma(m\theta+1)} y^{m\theta} (1-y)^{m(1-\theta)} \\
&\leq \sup_{\theta \in[\theta_-,\theta_+]} \sup_{y \in [0,1]}  \frac{e/2\pi}{ \sqrt{m \theta(1-\theta)}} \frac{y^{m\theta} (1-y)^{m(1-\theta)}}{\theta^{m\theta} (1-\theta)^{m(1-\theta)} } \\
&= \sup_{\theta \in[\theta_-,\theta_+]} \frac{e/2\pi}{ \sqrt{m \theta(1-\theta)}} \leq \frac{2}{ \sqrt{m \theta_l(1-\theta_l)}}  =: \gamma
\end{align*}
by Stirling's approximation:  $\sqrt{2\pi} \leq \frac{\Gamma(s+1)}{ e^{-s} s^{s+1/2}} \leq e$ \citep{spira1971calculation} and \eqref{theta_pm_inf}. And for any $y \in [\theta_0,\theta_1]$
\begin{align*}
 \hspace{.5in}&\hspace{-.5in}b( 2\eta(y)-\eta(\theta_*)) - (2b(\eta(y))-b(\eta(\theta_*))) \\
 &=   m\log(1+e^{2\eta(y) - \eta(\theta_*)} )  - 2m \log(1+e^{\eta(y)} ) +  m\log(1+e^{\eta(\theta_*)} ) \\
&=m\log\left(  \frac{(1+e^{2\eta(y) - \eta(\theta_*)} )(1+e^{\eta(\theta_*)} ) }{ (1+e^{\eta(y)} )^2 } \right) \\
&=m\log\left(  \left(1+\left( \frac{y}{1-y} \right)^2 \frac{1-\theta_*}{\theta_*} \right)\left(\frac{1}{1-\theta_*}\right) (1-y)^2 \right) \\
&=m\log\left(  ( 1-y)^2\frac{1}{1-\theta_*}+y^2 \frac{1}{\theta_*}  \right)  \\
&=m\log\left(  ( 1-2y+y^2)\frac{\theta_*}{\theta_*(1-\theta_*)}+y^2 \frac{1-\theta_*}{\theta_*(1-\theta_*)}  \right) \\
&=m\log\left(  ( 1-2y)\frac{\theta_*}{\theta_*(1-\theta_*)}+y^2 \frac{1}{\theta_*(1-\theta_*)}  \right) \\
&=m\log\left( 1 + \frac{(y-\theta_*)^2}{\theta_*(1-\theta_*)}  \right) 
\end{align*}
so 
\begin{align*}
&\sup_{y \in [\theta_0,\theta_1]} b( 2\eta(y)-\eta(\theta_*)) - (2b(\eta(y))-b(\eta(\theta_*))) \\
&\hspace{.5in}\leq \sup_{y \in [\theta_0,\theta_1]} m\log\left( 1 + \frac{(y-\theta_*)^2}{\theta_*(1-\theta_*)}  \right) \leq m\left( \frac{(\theta_1-\theta_0)^2}{\theta_*(1-\theta_*)}  \right)  =: \kappa.
\end{align*}

Noting that $M_2(\theta) = m \theta(1-\theta)$,
\begin{align*}
\sup_{y \in [\theta_0, \theta_1]} & M_2(y)^2  (2 + \gamma (\dot{b}(\eta(\theta_+))-\dot{b}(\eta(\theta_-))) ) \leq m^2 \left( \theta_h(1-\theta_h)\right)^2 (2 + \gamma m (\theta_+-\theta_-) )\\
&\leq m^2 \left( \theta_h(1-\theta_h)\right)^2   \left(2 + \frac{2 m }{ \sqrt{ m \theta_l (1-\theta_l)}} 30(\theta_1-\theta_0) \right) \\
&\leq m^2 \left( \theta_h(1-\theta_h)\right)^2  \left(2 + 60 \sqrt{m  \frac{(\theta_1-\theta_0)^2}{\theta_l(1-\theta_l)}}\right).
\end{align*}

Since for any $\theta \in [0,1]$ 
\begin{align*}
M_4(\theta) =  m \theta(1-\theta) \left( 3 \theta(1-\theta) (m-2) +1 \right)< 3 m^2 \left( \theta(1-\theta)\right)^2 + m  \theta(1-\theta),
\end{align*}
we have
\begin{align*}
8 M_4(\theta_- )& + 8 M_4( \theta_+ )  + 16\left( \dot{b}( \eta(\theta_+) ) - \dot{b}( \eta(\theta_-)) \right)^4  + \tfrac{2}{5} \gamma \left( \dot{b}(\eta(\theta_+))-\dot{b}(\eta(\theta_-))\right)^5\\ 
  \leq& 24 m^2 \left( \theta_-(1-\theta_-)\right)^2 + 8 m  \theta_-(1-\theta_-) + 24 m^2 \left( \theta_+(1-\theta_+)\right)^2 + 8 m  \theta_+(1-\theta_+) \\
  &+  16 m^4 ( \theta_+ - \theta_- )^4  + \frac{4/5}{\sqrt{m \theta_l(1-\theta_l)}} m^5  (\theta_+-\theta_-)^5  \\
  \leq& 4800 m^2 \left( \theta_h(1-\theta_h)\right)^2 + 160 m  \theta_h(1-\theta_h) \\
  &+ 3240000 m^4 ( \theta_1 - \theta_0 )^4  + 19440000 \sqrt{ m \frac{(\theta_1-\theta_0)^2}{\theta_l(1-\theta_l)}} m^4 (\theta_1-\theta_0)^4
 \end{align*}
 where we have applied \eqref{theta_pm_sup} and \eqref{theta_pm_diff}. Finally, recall from above that
 \begin{align*}
 \eta(\theta_1) - \eta(\theta_0) \leq \frac{\theta_1-\theta_0}{\theta_l(1-\theta_l)} \leq \frac{3}{2} \frac{\theta_1-\theta_0}{\theta_*(1-\theta_*)}.
 \end{align*}
\noindent\textbf{Step 3: Putting the pieces together}\\
 Noting that $\theta_l(1-\theta_l) \leq \theta_*(1-\theta_*) \leq \theta_h(1-\theta_h)$ and $\frac{\theta_h(1-\theta_h)}{\theta_l(1-\theta_l)} \leq 3/2$ by \eqref{theta_star_sup_inf}, we can use $\theta_*(1-\theta_*)$ throughout at the cost of a constant. 
Putting it altogether, if $m \frac{(\theta_1-\theta_0)^2}{\theta_*(1-\theta_*)} \leq 1$ then $\kappa \leq 1$ and 
\begin{align*}
c &\leq c' \left(  m^2 \left( \theta_*(1-\theta_*) \right)^2 + m \theta_*(1-\theta_*) + m^4 (\theta_1-\theta_0)^4  \right) \\
&\leq c' \left(  m^2 \left( \theta_*(1-\theta_*) \right)^2 + m \theta_*(1-\theta_*)  \right)
\end{align*}
for some absolute constant $c'$. Thus,
\begin{align*}
\hspace{1in}&\hspace{-1in}c \left( \tfrac{1}{2} \alpha(1-\alpha) \left( \eta(\theta_1)-\eta(\theta_0) \right)^2 \right)^2 \\
&\leq c'\left(  m^2 \left( \theta_*(1-\theta_*) \right)^2 + m \theta_*(1-\theta_*)  \right) \left( \frac{9}{8} \alpha(1-\alpha)  \frac{(\theta_1-\theta_0)^2}{(\theta_*(1-\theta_*))^2} \right)^2 \\
&\leq c' \left(m^2 + \frac{m}{\theta_*(1-\theta_*)} \right) \left( \frac{9}{8} \alpha(1-\alpha) \frac{(\theta_1-\theta_0)^2}{\theta_*(1-\theta_*)} \right)^2 .
\end{align*} 
}

\begin{remark}
We recall that if $\alpha, \theta_0, \theta_1$ are unknown, then any fixed sample strategy would not know how to pick $m$ sufficiently large a priori. Thus, the above corollary states that for any fixed $m$, whenever $\frac{(\theta_1-\theta_0)^2}{\theta_*(1-\theta_*)}$ is sufficiently small the number of samples necessary for this simple and intuitive strategy to identify the most biased coin scales like $\left(\frac{\theta_*(1-\theta_*)}{\alpha (\theta_1-\theta_0)^2 }\right)^2 \log(1/\delta)$. However, in the next section we show that when $\alpha,\theta_0,\theta_1$ are known and $m$ can be chosen by the algorithm, this same fixed sample strategy can identify the most biased coin using just $\frac{\log(1/(\alpha\delta))}{\alpha (\theta_1-\theta_0)^2}$ total flips in expectation, nearly matching the lower bound of Corollary~\ref{known_all_lower}. This is a striking example of the difference when parameters are known versus when they are not.
\end{remark}

\begin{table}[htdp]
\begin{center}
\begin{tabular}{|l|c|}\hline
\multicolumn{1}{|c|}{\textbf{Setting}} &  \textbf{Upper Bound} \\ \hline
Fixed algorithm, known $\alpha, \theta_0, \theta_1$ (Theorem~\ref{fixed_upper_known}) & $\frac{c\log(1/(\delta\alpha))}{\alpha(\theta_1 - \theta_0)^2}$ \rowpad{1.3em}{-0.8em} \\ \hline
Adaptive algorithm, known $\alpha, \theta_0, \theta_1$ & \\(\cite{chandrasekaran2014finding,malloy2012quickest}) & $ \frac{c}{(\theta_1 - \theta_0)^2} \left( \frac{1}{\alpha} + \log(\frac{1}{\alpha\delta}) \right)$ \rowpad{1.3em}{-0.8em} \\ \hline
Adaptive algorithm, unknown $\theta_0, \theta_1$ (Theorem~\ref{unknown_epsilon}) & $\frac{c\log\left(\log\left(\frac{1}{(\theta_1 - \theta_0)^2}\right) / \delta\right)}{\alpha(\theta_1-\theta_0)^2}$ \rowpad{2.2em}{-0.8em} \\ \hline
Adaptive algorithm, unknown $\alpha$ (Theorem~\ref{unknown_alpha}) & $\frac{c\log(\log(1 / \alpha)/\delta)}{\alpha(\theta_1-\theta_0)^2}$ \rowpad{1.3em}{-0.8em} \\ \hline
Adaptive algorithm, unknown $\alpha, \theta_0, \theta_1$ (Theorem~\ref{unknown_all}) & $\frac{c\log\left(\frac{1}{\alpha(\theta_1 - \theta_0)^2}\right)\log\left(\log\left(\frac{1}{\alpha(\theta_1 - \theta_0)^2}\right) / \delta\right)}{\alpha(\theta_1 - \theta_0)^2}$ \rowpad{2.5em}{-1em} \\ \hline
\end{tabular}
\end{center}
\vspace{-0.5cm}
\caption{Upper bounds on the expected sample complexity of algorithms that identify a heavy coin with probability at least $1-\delta$ under different states of prior knowledge. Recall the \emph{fixed} algorithm samples each coin exactly $m$ times, for some fixed $m \in \mathbb{N}$. Also note that the algorithms of Section~\ref{sec:upper_bounds} apply to general distributions beyond just coins.}
\label{tab:upper-bounds}
\end{table}
\vspace{-0.8cm}

\section{Upper bounds and algorithms}\label{sec:upper_bounds}
Above we presented lower bounds on the difficulty of identifying a heavy distribution. 
In this section we prove the existence of algorithms that nearly match the lower bounds, even with only partial side knowledge.
Table~\ref{tab:upper-bounds} summarizes the algorithms and their bounds. 
Our main result in Section~\ref{subsec:upper_adaptive_unknown} is Theorem~\ref{unknown_all} which describes the performance of an algorithm that has no prior knowledge of the parameters $\alpha,\theta_0,\theta_1$ yet yields an upper bound that matches the lower bound of Theorem~\ref{lower_adaptive_known} up to logarithmic factors.
In what follows, we assume that samples from heavy or light distributions are supported on $[0, 1]$, and that drawn samples are independent and unbiased estimators of the mean, i.e., $\E[ X_{i,j} ] = \mu_i$ for $\mu_i \in \{\theta_0, \theta_1\}$. All results can be easily extended to sub-Gaussian distributions.
We begin with a fixed sample strategy and then turn our attention to adaptive sampling procedures.


\subsection{Fixed sample strategy for known $\alpha,\theta_0,\theta_1$}
\label{subsec:upper_fixed_known}
A lower bound on $\alpha$ tells us how many distributions we must consider and knowledge of the difference $(\theta_1-\theta_0)$ tells us how many times we should sample each distribution. The below theorem comes within a $\log(1/\delta)$ factor of the lower bound proved in Corollary~\ref{known_all_lower} in general and is tight when $\alpha \leq \delta$.

\begin{theorem}[Fixed sample size, known $\alpha$ and $\theta_0,\theta_1$]\label{fixed_upper_known}
Fix $\delta \in (0,1/4)$ and set  $\widehat{n} = \left\lceil\tfrac{1}{\alpha} \log( \tfrac{2}{\delta} )\right\rceil$ and $m = \left\lceil\frac{2 \log( 4 \widehat{n} / \delta )}{(\theta_1-\theta_0)^2}\right\rceil$. There exists a fixed sample size strategy with stopping time $N_m \leq \widehat{n}$ that is $\delta$-probably correct and satisfies
\begin{align*}
\E[ m N_m ] \leq  3 \frac{  \log( 1 /\alpha)+ \log(12 \log( 6 / \delta ) /\delta)}{\alpha (\theta_1-\theta_0)^2} \leq 12 \frac{ \log( \tfrac{2}{\delta \alpha} ) }{\alpha (\theta_1-\theta_0)^2}.
\end{align*}
\end{theorem}

\makeproof{FixedUpperKnownProof}{Let $\widehat{\mu}_i$ be the empirical mean of the $i$th distribution sampled $m$ times with mean $\mu_i \in \{\theta_{0},\theta_1\}$. Let $N$ be the minimum of $\widehat{n}$ and the first $i \in \mathbb{N}$ such that $\widehat{\mu}_i \geq \frac{\theta_0+\theta_1}{2}$. Declare distribution $N$ to be heavy. The total number of flips this procedure makes equals $m N$. 

Define the events 
\begin{align*}
\xi_1 = \bigcup_{i=1}^{\widehat{n}} \{ \mu_i = \theta_1 \} , \quad \text{ and } \quad \xi_2 = \bigcap_{i=1}^{\widehat{n}} \{ | \widehat{\mu}_i-\mu_i| < \tfrac{\theta_1-\theta_0}{2} \} .
\end{align*}
Note that $\P( \xi_1^c ) = \P( \mu_1 = \theta_0 )^{\widehat{n}} = (1-\alpha)^{\widehat{n}} \leq \exp( - \alpha \widehat{n}) \leq \delta/2$. And, by a union bound and Chernoff's inequality $\P\left( \xi_2^c \right) \leq 2 \widehat{n} e^{-m (\theta_1-\theta_0)^2 / 2} \leq \delta/2$. Thus, the probability that $\xi_1$ or $\xi_2$ fail to occur is less than $\delta$, so in what follows assume they succeed. 

Under $\xi_1$ at least one of the $\widehat{n}$ distributions is heavy.
Under $\xi_2$, for any $i \in [\widehat{n}]$ with $\mu_i = \theta_0$ we have $\hat{\mu}_i < \mu_i + \frac{\theta_1-\theta_0}{2} = \frac{\theta_0+\theta_1}{2}$ which implies that the procedure will never exit with a light distribution unless $N = \widehat{n}$. On the other hand, for the first $i \in [\widehat{n}]$ with $\mu_i=\theta_1$ we have $\hat{\mu}_i > \mu_i - \frac{\theta_1-\theta_0}{2} = \frac{\theta_0+\theta_1}{2}$ which means the algorithm will output distribution $i$ at time $N=i$. Thus, $N$ is equal to the first distribution that is heavy and
\begin{align*}
\E[ N ] &= \sum_{n=1}^{\widehat{n}} \P( N \geq n ) =  \sum_{n=1}^{\widehat{n}} \P( N \geq n,  \max_{i=1,\dots,n-1} \mu_i \neq \theta_1 ) + \P( N \geq n,  \max_{i=1,\dots,n-1} \mu_i = \theta_1 )  \\
&\leq  \sum_{n=1}^{\widehat{n}} \P(  \max_{i=1,\dots,n-1} \mu_i \neq \theta_1 ) + \P( \cup_{i=1}^{n-1} \{ | \widehat{\mu}_i-\mu_i| > \tfrac{\theta_1-\theta_0}{2} \} | \max_{i=1,\dots,n-1} \mu_i = \theta_1 ) \P(\max_{i=1,\dots,n-1} \mu_i = \theta_1) \\
&\leq  \sum_{n=1}^{\widehat{n}} \P(  \max_{i=1,\dots,n-1} \mu_i \neq \theta_1 ) + \P( \cup_{i=1}^{n-1} \{ | \widehat{\mu}_i-\mu_i| > \tfrac{\theta_1-\theta_0}{2} \}  )  \\
&\leq  \sum_{n=1}^{\widehat{n}} (1-\alpha)^{n-1} + \tfrac{n-1}{\widehat{n}}\frac{\delta}{2}  \leq \frac{1}{\alpha} + \widehat{n} \delta /4 = \frac{1}{\alpha} ( 1+ \tfrac{\delta \log(2e/\delta)}{4} ) \leq \frac{3/2}{\alpha}.
\end{align*}
Multiplying $\E[N]$ by $m$ yields the result.}

\subsection{Fully adaptive strategies when $\alpha$ and/or $\theta_0,\theta_1$ are known}
\label{subsec:upper_adaptive_known}
While the previous section considered a strategy that takes a constant number of samples from each distribution, this section allows the procedure to determine the number of times to sample a particular distribution adaptively based on the samples from that distribution.
This section also shows that there exist simple procedures that adapt to the case when only a subset of $\alpha,\theta_0,\theta_1$ are known using just a small number of samples more than if they had been known.

Consider Algorithm~\ref{alg:upper-adaptive-bounded}, an SPRT-like procedure for finding a heavy distribution given $\delta$ and lower bounds on $\alpha$ and $\epsilon$.

\begin{algorithm}
\begin{framed}
\textbf{Given} $\delta \in (0, 1/4), \alpha_0 \in (0, 1/2), \epsilon_0 \in (0, 1)$. \\
\textbf{Initialize} $n=\lceil2\log(9)/\alpha_0\rceil, m = \lceil64\epsilon_0^{-2}\log(14n/\delta)\rceil, A = -8\epsilon_0^{-1}\log(21)$,\\ 
\forceindent $B = 8\epsilon_0^{-1}\log(14n/\delta), k_1=5, k_2=\lceil8\epsilon_0^{-2}\log(2k_1/\min\{\delta/8, m^{-1}\epsilon_0^{-2}\})\rceil$. \\
\textbf{Draw} $k_1$ distributions and sample them each $k_2$ times.\\
\textbf{Estimate} $\widehat{\theta}_0 = \min_{i = 1,\ldots,k_1} \widehat{\mu}_{i, k_2}, \hat{\gamma} = \widehat{\theta}_0 + \epsilon_0/2$.\\
\textbf{Repeat} for $i = 1, \ldots, n$: \\
    \forceindent \textbf{Draw} distribution $i$. \\
    \forceindent \textbf{Repeat} for $j = 1, \ldots, m$: \\
        \forceindent\forceindent\textbf{Sample} distribution $i$ and observe $X_{i,j}$. \\
        \forceindent\forceindent\textbf{If} $\sum_{k=1}^j (X_{i,k} - \hat{\gamma}) > B$:\\ 
            	\forceindent\forceindent\forceindent \textbf{Declare} distribution $i$ to be heavy and \textbf{Output} distribution $i$. \\
	\forceindent\forceindent\textbf{Else if} $\sum_{k=1}^j (X_{i,k} - \hat{\gamma}) < A$:\\
		\forceindent\forceindent\forceindent \textbf{break}. \\
\textbf{Output} \texttt{null}.
\end{framed}
\caption{Adaptive strategy for heavy distribution identification with inputs $\alpha_0,\epsilon_0, \delta$}
\label{alg:upper-adaptive-bounded}
\end{algorithm}

\begin{theorem}\label{upper-adaptive-bounded}
If Algorithm~\ref{alg:upper-adaptive-bounded} is run with $\delta \in (0, 1/4), \alpha_0 \in (0, 1/2), \epsilon_0 \in (0, 1)$, then the expected number of total samples taken by the algorithm is no more than
\begin{align*}
 \frac{c' \alpha \log(1/\alpha_0) + c''\log\left(\frac{1}{\delta}\right)}{\alpha_0\epsilon_0^2}
\end{align*}
for some absolute constants $c'$,$c''$, and all of the following hold: 1) with probability at least $1 - \delta$, a light distribution is not returned, 2) if $\epsilon_0 \leq \theta_1 - \theta_0$ and $\alpha_0 \leq \alpha$, then with probability $\frac{4}{5}$ a heavy distribution is returned, and 3) the procedure takes no more than $ \frac{c\log(1/(\alpha_0\delta))}{\alpha_0\epsilon_0^2}$ total samples.
\end{theorem}

\makeproof*{KnownAllProof}{

First, we prove several technical lemmas necessary to analyze our algorithm.

\begin{lemma} \label{linear_concentration}
For $i \in \mathbb{N}$, let $X_i \in [a_i,b_i]$ for $|b_i-a_i| \leq 1$ be a random variable with $\E[X_i]=0$. Then 
\begin{align*}
\P\left( \bigcup_{n=1}^\infty \left\{ \sum_{i=1}^n X_i \geq \alpha n + \beta \right\} \right) \leq 7\exp(- \alpha\beta /2 )
\end{align*}
whenever $\alpha\beta \geq 1$.
\end{lemma}
\begin{proof}
First we will break the bound into two pieces:
\begin{align*}
\P\left( \bigcup_{n=1}^\infty \left\{ \sum_{i=1}^n X_i \geq \alpha n + \beta \right\} \right)  \leq \min_{n_0} \P\left( \bigcup_{n=1}^{n_0} \left\{ \sum_{i=1}^{n} X_i \geq \beta \right\} \right)  + \P\left( \bigcup_{n=n_0+1}^\infty \left\{ \sum_{i=1}^n X_i \geq \alpha n \right\} \right) 
\end{align*}
where $\P\left( \bigcup_{n=1}^{n_0} \left\{ \sum_{i=1}^{n} X_i \geq \beta \right\} \right) \leq \exp( -2 \beta^2 / n_0 )$ by Doob-Hoeffding's maximal inequality.
For any fixed $k \in \mathbb{N}$:
\begin{align*}
\P\left( \sum_{i=1}^{2^k} X_i \geq \alpha 2^k/2  \right) &\leq \exp( - \alpha^2 2^{k} /2) 
\end{align*}
and
\begin{align*}
\P\left( \bigcup_{n=2^k+1}^{2^{k+1}} \left\{ \sum_{i=2^k+1}^n X_i \geq \alpha n / 2  \right\} \right) &\leq \P\left( \bigcup_{n=2^k+1}^{2^{k+1}} \left\{ \sum_{i=2^k+1}^n X_i \geq \alpha 2^k / 2  \right\} \right) \\
&= \P\left( \bigcup_{\ell=1}^{2^k} \left\{ \sum_{i=1}^\ell X_i \geq \alpha 2^k / 2  \right\} \right) \leq \exp( - \alpha^2 2^k /2)
\end{align*}
by Hoeffding's and Doob-Hoeffding's maximal inequality, respectively.
Thus 
\begin{align*}
\P\left( \bigcup_{n=n_0}^\infty \left\{ \sum_{i=1}^{n} X_i \geq \alpha n \right\} \right) &= \P\left( \bigcup_{n=n_0}^\infty \left\{ \sum_{i=1}^{2^{\lceil \log_2(n)\rceil}} X_i  + \sum_{i=2^{\lceil \log_2(n)\rceil} + 1}^{n} X_i \geq \alpha n \right\} \right)  \\
&=  \P\left( \bigcup_{k=\log_2(n_0)}^\infty \bigcup_{n=2^k+1}^{2^{k+1}} \left\{ \sum_{i=1}^{2^k} X_i  + \sum_{i=2^k + 1}^{n} X_i \geq \alpha n \right\} \right)  \\
&\leq  \sum_{k=\log_2(n_0)}^\infty \P\left(  \sum_{i=1}^{2^k} X_i \geq \alpha 2^k/2 \right) + \P\left( \bigcup_{n=2^k+1}^{2^{k+1}} \left\{  \sum_{i=2^k + 1}^{n} X_i \geq \alpha n/2 \right\} \right)  \\
&\leq \sum_{k=\log_2(n_0)}^\infty 2 \exp( - \alpha^2 2^{k} /2) \leq 2 \int_{\log_2(n_0)}^\infty \exp(-(\alpha/2)^2 2^x) dx \\
&= \frac{2}{\log(2)} \int_{n_0}^\infty u^{-1} \exp( -(\alpha/2)^2 u) du \leq \frac{8 \exp(-(\alpha/2)^2 n_0)}{n_0 \alpha^2 \log(2)}. 
\end{align*}
Putting the pieces together we have
\begin{align*}
\P\left( \bigcup_{n=1}^\infty \left\{ \sum_{i=1}^n X_i \geq \alpha n + \beta \right\} \right)  &\leq \min_{n_0} \exp( -2 \beta^2 / n_0 ) + \frac{8 \exp(-(\alpha/2)^2 n_0)}{n_0 \alpha^2 \log(2)} \\
&\leq \exp( -\beta \alpha ) + \frac{4 \exp(- \beta \alpha/2 )}{\beta \alpha \log(2)} \leq 7\exp(-\beta \alpha /2 )
\end{align*}
where the last inequality holds with $\beta\alpha \geq 1$.
\end{proof}

\begin{lemma}\label{upper_adaptive_bounded_lemma_theta1}
Given $\theta_1-\hat{\gamma} \ge \frac{2B}{m}$,
\begin{align*}
&\P\left(\max_{j=1,\ldots,m}\sum_{s=1}^j (X_{i,s} - \hat{\gamma}) > B\big|\mu_i = \theta_1\right) \geq 1 - \exp\left(-m(\theta_1 - \hat{\gamma})^2/2\right).
\end{align*}
Similarly, given $\hat{\gamma} - \theta_0 \ge \frac{2|A|}{m}$,
\begin{align*}
&\P\left(\min_{j=1,\ldots,m}\sum_{s=1}^j (X_{i,s} - \hat{\gamma}) < A\big|\mu_i = \theta_0\right) \geq 1 - \exp\left(-m(\hat{\gamma} - \theta_0)^2/2\right).
\end{align*}
\end{lemma}

\begin{proof}
We analyze the left hand side of the lemma:
\begin{align*}
\P&\left(\max_{j=1,\ldots,m}\sum_{s=1}^j (X_{i,s} - \hat{\gamma}) > B\bigg| \mu_i = \theta_1 \right) \\
&= \P\left(\bigcup_{j=1}^m\left\{\sum_{s=1}^j (X_{i,s} - \hat{\gamma}) > B\bigg| \mu_i = \theta_1 \right\}\right) \\
&\geq \P\left(\sum_{s=1}^m (X_{i,s} - \hat{\gamma}) > B\bigg| \mu_i = \theta_1 \right) \\
&= 1 - \P\left(\frac{1}{m}\sum_{s=1}^m (X_{i,s} - \mu_i) \leq \frac{B}{m} - (\mu_i - \hat{\gamma})\bigg| \mu_i = \theta_1 \right) \\
&= 1 - \P\left(\frac{1}{m}\sum_{s=1}^m (\mu_i - X_{i,s}) \geq (\mu_i - \hat{\gamma}) - \frac{B}{m} \bigg| \mu_i = \theta_1 \right) \\
&\geq 1 - \exp\left(-2m\left[(\theta_1 - \hat{\gamma}) - \frac{B}{m}\right]^2\right) \\
&\geq 1 - \exp\left(\frac{-m(\theta_1 - \hat{\gamma})^2}{2}\right)
\end{align*}
Where the second to last statement holds by Hoeffding's inequality, and the last uses the bound on $B/m$ given in the lemma. A nearly identical argument yields the second half of the lemma.
\end{proof}

\begin{lemma}\label{lemma:upper-adaptive-bounded-correct}
If $\theta_1 - \hat{\gamma} \ge \frac{2B}{m}$ then
\begin{align*}
\P\left(\bigcup_{i=1}^n \left\{ \mu_i = \theta_1 , \max_{j=1,\ldots,m}\sum_{s=1}^j (X_{i,s} - \hat{\gamma}) > B, \min_{j=1,\ldots,m}\sum_{s=1}^j (X_{i,s} - \hat{\gamma}) > A \right\} \right) \\
\geq 1- \exp\left[ - \alpha n (1 - \exp\left(-B(\theta_1 - \hat{\gamma})\right)  - 7 \exp( -|A|(\hat{\gamma} - \theta_0)/2 )) \right]
\end{align*}
\end{lemma}
\begin{proof}
Consider iid events $\Omega_i$ for $i=1,\dots,n$. Then $\P(\bigcup_{i=1}^n \Omega_i ) = 1-\P( \bigcap_{i=1}^n \Omega_i^c ) = 1 - \P( \Omega_1^c )^n = 1-(1-\P(\Omega_i))^n \geq 1- \exp(- n\P(\Omega_i) )$. We follow the same line of reasoning:
\begin{align*}
\P&\left(\bigcup_{i=1}^n \left\{ \mu_i = \theta_1 , \max_{j=1,\ldots,m}\sum_{s=1}^j (X_{i,s} - \hat{\gamma}) > B, \min_{j=1,\ldots,m}\sum_{s=1}^j (X_{i,s} - \hat{\gamma}) > A \right\} \right) \\
&= 1 - \left( 1- \P\left(\mu_i = \theta_1 , \max_{j=1,\ldots,m}\sum_{s=1}^j (X_{i,s} - \hat{\gamma}) > B, \min_{j=1,\ldots,m}\sum_{s=1}^j (X_{i,s} - \hat{\gamma}) > A   \right) \right)^n \\
&= 1 - \left( 1- \alpha \P\left( \max_{j=1,\ldots,m}\sum_{s=1}^j (X_{i,s} - \hat{\gamma}) > B, \min_{j=1,\ldots,m}\sum_{s=1}^j (X_{i,s} - \hat{\gamma}) > A \Big| \mu_i = \theta_1  \right) \right)^n \\
&= 1 - \left( 1- \alpha \left( 1 - \P\left( \max_{j=1,\ldots,m}\sum_{s=1}^j (X_{i,s} - \hat{\gamma}) < B \Big| \mu_i = \theta_1  \right) - \P\left(\min_{j=1,\ldots,m}\sum_{s=1}^j (X_{i,s} - \hat{\gamma}) < A \Big| \mu_i = \theta_1  \right) \right)\right)^n \\
&\geq 1 - \left( 1- \alpha \left( 1 - \exp\left(-m(\theta_1 - \hat{\gamma})^2/2\right)  - 7 \exp( -|A|(\hat{\gamma} - \theta_0)/2 ) \right)\right)^n \\
&\geq 1- \exp\left[ - \alpha n (1 - \exp\left(-m(\theta_1 - \hat{\gamma})^2/2\right)  - 7 \exp( -|A|(\hat{\gamma} - \theta_0)/2 )) \right] \\
&\geq 1- \exp\left[ - \alpha n (1 - \exp\left(-B(\theta_1 - \hat{\gamma})\right)  - 7 \exp( -|A|(\hat{\gamma} - \theta_0)/2 )) \right] 
\end{align*}
Where the third-to-last inequality applies Lemmas~\ref{upper_adaptive_bounded_lemma_theta1} and \ref{linear_concentration}. 
\end{proof}

Now, we are ready to prove Theorem~\ref{upper-adaptive-bounded}.
\vspace{0.3cm}
\begin{proof}
First, we consider the estimation of $\widehat{\theta}_0$ of  Algorithm~\ref{alg:upper-adaptive-bounded}, then consider the sample complexity of the algorithm, and then prove correctness.

Let $\xi_0 = \{ \widehat{\theta}_0 - \theta_0 \geq -\frac{\epsilon_0}{4} \}$ and $\xi_1 = \{ \widehat{\theta}_0 - \theta_0 \leq \frac{\epsilon_0}{4} \}$ be the events that we accurately estimate the parameter $\theta_0$. We will show that $\P(\xi_0) \geq 1-\delta'$ and $\P(\xi_1) \geq 3/4$ where $\delta' = \min\{\delta/8,\frac{1}{m\epsilon_0^2}\}$. Let $k_1=5$ and $k_2 = 8 \epsilon_0^{-2} \log( \frac{2k_1}{\delta'})$.
First note that
\begin{align*}
\P\left( \bigcup_{i=1}^{k_1} \left\{  |\widehat{\mu}_{i,k_2} - \mu_i| \geq \frac{\epsilon_0}{4} \right\} \right) &\leq 2 k_1 \exp( -2 k_2 (\epsilon_0/4)^2 ) \leq \delta' 
\end{align*}
so that with probability at least $1-\delta'$ we have $\widehat{\theta}_0 = \min_{i=1,\dots,k_1} \widehat{\mu}_{i,k_2} \geq \min_{i=1,\dots,k_1} \mu_i - \epsilon_0 /4 \geq \theta_0 - \epsilon_0 / 4$, and in particular, $\P(\xi_0) \geq 1-\delta'$. Let $\mathcal{E} = \{\bigcup_{i=1}^{k_1} \{ \mu_i = \theta_0 \} \}$ be the event that at least one of the distributions is light. Then 
\begin{align*}
\P\left( \mathcal{E} \right) = 1- \alpha^{k_1} \geq 1-  2^{-k_1} \geq 31/32 , 
\end{align*}
so that under $\mathcal{E} \cap \  \xi_0$, we have $\widehat{\theta}_0 = \min_{i=1,\dots,k_1} \widehat{\mu}_{i,k_2} \leq \min_{i=1,\dots,k_1} \mu_i + \epsilon_0 /4 = \theta_0 + \epsilon_0 /4$ which means $\P( \xi_1^c ) \leq \P(\xi_0^c \cup \mathcal{E}^c) \leq \delta/8+1/32 \leq 1/16$.
Moreover, the total number of samples is bounded by $k_1 k_2 = c \epsilon_0^{-2} \log(1/\delta') \leq c \epsilon_0^{-2} \log(\max\{\frac{1}{\delta},\log(\frac{1}{\alpha_0\delta})\})$ which is clearly dominated by $\frac{\log(1/\delta)}{\alpha_0 \epsilon_0^2}$.

We now turn our attention to the sample complexity. By Wald's identitity~\citep[Proposition 2.18]{siegmund2013sequential},
\begin{align*}
\E[T] &= \E\left[\sum_{i=1}^N M_i\right] = \E[N]\E[M_1] = \E[N]((1-\alpha)\E[M_1 | \mu_1 = \theta_0] + \alpha\E[M_1 | \mu_1 = \theta_1]).
\end{align*}
Trivially, $\E[N] \leq n$ and $\E[M_1 | \mu_1 = \theta_1] \leq m$, so we only need to bound $\E[M_1 | \mu_1 = \theta_0]$. Clearly we have that
\begin{align*}
\E[M_1 | \mu_1 = \theta_0] = \E[M_1 | \xi_0, \mu_1 = \theta_0 ] \P(\xi_0) + \E[M_1 | \xi_0^c, \mu_1 = \theta_0] \P(\xi_0^c) \leq \E[M_1 | \xi_0, \mu_1 = \theta_0 ] + \delta'm
\end{align*}
so
\begin{align*}
\E&[M_1 | \xi_0, \mu_1 = \theta_0] \leq \sum_{t=1}^\infty \P\left(\argmin_j \left\{ \sum_{s=1}^j (X_{1,s} - \hat{\gamma}) < A \Big| \xi_0, \mu_1 = \theta_0\right\} \geq t\right)\\
&= \sum_{t=1}^\infty 1 - \P\left(\min_{j=1,\ldots,t-1}\sum_{s=1}^j (X_{1,s} - \hat{\gamma}) < A \Big| \xi_0, \mu_1 = \theta_0 \right) \\
&= \sum_{t=0}^\infty 1 - \P\left(\min_{j=1,\ldots,t}\sum_{s=1}^j (X_{1,s} - \hat{\gamma}) < A \Big| \xi_0, \mu_1 = \theta_0 \right) \\
&\leq \sum_{t=0}^\infty  1- \1_{\hat{\gamma} - \theta_0 \ge \frac{2|A|}{t}} (1- \exp\left(-t(\hat{\gamma} - \theta_0)^2/2\right) \\
&\leq \frac{2|A|}{\hat{\gamma} - \theta_0} + 2 e^{1/2}(\hat{\gamma} - \theta_0)^{-2}  \exp\left(-|A|(\hat{\gamma} - \theta_0)\right) \leq \frac{3|A|}{\hat{\gamma} - \theta_0} \leq \frac{293}{\epsilon_0^2}.
\end{align*}
where the second inequality follows by applying Lemma~\ref{upper_adaptive_bounded_lemma_theta1} and the last inequality holds by $\xi_0$ and the value of $|A|$ since if $\xi_0$ holds, $\hat{\gamma} - \theta_0 = \widehat{\theta}_0 - \theta_0 + \frac{\epsilon_0}{2} \geq \frac{\epsilon_0}{4}$.
Thus
\begin{align*}
\E[M_1] \leq (1-\alpha)\left[\left(\frac{293}{\epsilon_0^2}\right) + \delta'm\right] + \alpha m \leq \delta' m + \frac{1}{\epsilon_0^2}\left(293 + 64\alpha\log\left(\tfrac{14n}{\delta}\right)\right) \leq \frac{c\alpha\log\left(\frac{1}{\alpha_0\delta}\right)}{\epsilon_0^2}
\end{align*}
for some $c$ where we use the fact that $\delta'm \leq \epsilon_0^{-2}$.
So we have
\begin{align*}
\E[T] &\leq n\E[M_1] \leq \frac{c' \alpha \log(1/\alpha_0) + c''\log\left(\frac{1}{\delta}\right)}{\alpha_0\epsilon_0^2}.
\end{align*}

Now, we analyze the correctness claims.
Under $\xi_0$, $\hat{\gamma} - \theta_0 \geq \frac{\epsilon_0}{4}$. Note that this event fails to occur with probability less than $\delta/2$, and if it is used in conjunction with some other event that fails to occur with probability $\delta/2$, we may conclude that either of these events fail with probability less than $\delta$.

To justify Claim 1, we apply Lemma~\ref{linear_concentration} to observe that the probability that we output a light distribution is no greater than
\begin{align*}
\P&(\xi_0^c) + \P\left( \bigcup_{i=1}^n \left\{ \max_{j=1,\ldots,m}\sum_{s=1}^j (X_{i,s} - \hat{\gamma}) > B , \mu_i = \theta_0 \right\} \Big| \xi_0 \right) \P(\xi_0) \\
&\leq \P(\xi_0^c) + n  (1-\alpha) \P\left(\max_{j=1,\ldots,m}\sum_{s=1}^j (X_{i,s} - \hat{\gamma}) > B\big|\mu_i = \theta_0, \xi_0\right)\\
&\leq \delta/2+ 7 n  \exp( - B (\hat{\gamma}-\theta_0)/2) \leq \delta
%
\end{align*}
where we have used $\hat{\gamma} - \theta_0 \geq \frac{\epsilon_0}{4}$ and plugged in the values of $B$ and $n$.

To justify Claim 2, assume $\alpha_0 \leq \alpha$ and $\epsilon_0 \leq \theta_1 - \theta_0$. We apply Lemma~\ref{lemma:upper-adaptive-bounded-correct} to observe that the probability that we return a heavy distribution is at least
\begin{align*}
\hspace{0in}&\hspace{-0in}\P\left( \xi_0 \cap \xi_1 \cap \bigcup_{i=1}^n \left\{ \mu_i = \theta_1 , \max_{j=1,\ldots,m}\sum_{s=1}^j (X_{i,s} - \hat{\gamma}) > B, \min_{j=1,\ldots,m}\sum_{s=1}^j (X_{i,s} - \hat{\gamma}) > A \right\} \right) \\
&= \P(\xi_0 \cap \xi_1) \P\left( \bigcup_{i=1}^n \left\{ \mu_i = \theta_1 , \max_{j=1,\ldots,m}\sum_{s=1}^j (X_{i,s} - \hat{\gamma}) > B, \min_{j=1,\ldots,m}\sum_{s=1}^j (X_{i,s} - \hat{\gamma}) > A \right\} \bigg| \xi_0,\xi_1\right)  \\ 
&\geq \P(\xi_0 \cap \mathcal{E} )(1-\exp\left[ - \alpha n (1 - \exp\left(-B(\epsilon_0/4)\right)  - 7 \exp( -|A|(\epsilon_0/4)/2 )) \right]) \\
&\geq (15/16)(1-\exp\left[ - \alpha n (1 - (\tfrac{\delta}{14n})^2  - 1/3) \right] )\geq (15/16)(8/9) \geq 4/5
\end{align*}
where we have used $\P(\xi_0 \cap \mathcal{E} ) \geq 1 - \P(\xi_0^c) - \P(\mathcal{E}^c) \geq 15/16$, $(\tfrac{\delta}{14n})^2\leq 1/6$, $\alpha n \geq 2\log(9)$ and plugged in the values for $A$ and $B$. 

To justify Claim 3, we simply observe that the algorithm always terminates after $n \times m$ steps.
\end{proof}
}

Clearly, Theorem~\ref{upper-adaptive-bounded} applies when $\alpha,\theta_0,\theta_1$ are known. 
The third claim of the theorem follows from a trivial bound of $nm$ for the values of $n$ and $m$ stated in the algorithm (i.e. it holds with probability 1). The second claim holds only with constant probability (versus with probability $1-\delta$) since the probability of observing a heavy distribution among the $n=\lceil2\log(4)/\alpha_0\rceil$ distributions considered only occurs with constant probability. One can boost this probability to $1-\delta$ by repeated application of the algorithm $\log(1/\delta)$ times or alternatively, one can run the algorithm with $n=\Theta(\frac{\log(1/\delta)}{\alpha})$ (with a straightforward modification of the proof). Moreover, with a slightly more sophisticated argument, one can show that if the algorithm is run with $\widehat{\theta}_0=\theta_0$ (and the estimation step is skipped) and $n=\infty$ then the algorithm is nearly equivalent to the SPRT of \cite{malloy2012quickest} which succeeds with probability at least $1-\delta$ and achieves an expected sample complexity equivalent to \eqref{sprt_sample_complexity}.

\noindent\begin{minipage}[T]{\textwidth}
\noindent\begin{minipage}[c]{.5\textwidth}
\begin{algorithm}[H]
\begin{framed}
\textbf{Given} $\delta \in (0, 1), \alpha \in (0, 1/2)$. \\
\textbf{Initialize $k=1$}\\
\textbf{While} Algorithm~\ref{alg:upper-adaptive-bounded} run with inputs $\delta/(2k^2)$, \\ 
  \forceindent[0.75em]$\alpha_0=\alpha,\epsilon_0=2^{-k}$ returns \texttt{null}: \\
    \forceindent \textbf{Set} $k = k+1$.  \\
\textbf{Output} distribution $k$.
\end{framed}
\caption{Algorithm for unknown $\theta_1 -\theta_0$.}
\label{alg:unknown_epsilon}
\end{algorithm}
\end{minipage}
\noindent\begin{minipage}[c]{.5\textwidth}
\begin{algorithm}[H]
\begin{framed}
\textbf{Given} $\delta \in (0, 1), \epsilon \in (0, 1]$. \\
\textbf{Initialize $k=1$}\\
\textbf{While} Algorithm~\ref{alg:upper-adaptive-bounded} run with inputs $\delta/(2k^2)$, \\ 
  \forceindent[0.75em]$\alpha_0=2^{-k},\epsilon_0=\epsilon$ returns \texttt{null}: \\
    \forceindent \textbf{Set} $k = k+1$.  \\
\textbf{Output} distribution $k$.
\end{framed}
\caption{Algorithm for unknown $\alpha$.}
\label{alg:unknown_alpha}
\end{algorithm}
\end{minipage}
\end{minipage}
\vspace{.4cm}

We now leverage Theorem~\ref{upper-adaptive-bounded} to design procedures that do not have knowledge of these parameters using the ``doubling trick.''.
First we consider the case when $\alpha$ is known but a lower bound on $\theta_1-\theta_0$ is not. 
The following theorem characterizes the performance of Algorithm~\ref{alg:unknown_epsilon}.

\begin{theorem}[Known $\alpha$, unknown $\theta_0,\theta_1$] \label{unknown_epsilon}
Fix $\delta \in (0,1)$. If Algorithm~\ref{alg:unknown_epsilon} is run with $\delta, \alpha$ then with probability at least $1-\delta$ a heavy distribution is returned and the expected number of total samples taken is no more than
\begin{align*}
\frac{c  \log\left( \log\left(\tfrac{1}{(\theta_1-\theta_0)^2}\right) / \delta \right)  }{\alpha(\theta_1-\theta_0)^2}.
\end{align*}
for an absolute constant $c$. 
\end{theorem}

\makeproof{UnknownEpsilonProof}{
On each stage $k$, Algorithm~\ref{alg:upper-adaptive-bounded} is called with $\delta/(2k^2)$. By the guarantees of Theorem~\ref{upper-adaptive-bounded}, the probability that Algorithm~\ref{alg:unknown_epsilon} ever outputs a light distribution is less than $\sum_{k=1}^\infty \delta/(2k^2) \leq \delta$. Thus, if a distribution is output, it is heavy with probability at least $1-\delta$. We now show that the expected number of samples taken before outputting a distribution is bounded. 

Let $K$ be the random stage in which Algorithm~\ref{alg:unknown_epsilon} outputs a distribution and let $k_*$ be the smallest $k\in\mathbb{N}$ that satisfies $2^{-k} \leq \theta_1-\theta_0$. By the guarantees of Theorem~\ref{upper-adaptive-bounded} and the independence of the stages $k$, $\P(K\geq k_*+i) \leq \sum_{\ell=i}^\infty (\tfrac{1}{5})^\ell = (\tfrac{5}{4}) (\tfrac{1}{5})^i$. Moreover, if $M_k$ is the number of measurements taken at stage $k$, then by Wald's identity the expected number of measurements is bounded by
\begin{align*}
\E\left[ \sum_{k=1}^K M_k \right] &= \sum_{k=1}^\infty \E[N_k] \P(K\geq k) \leq \sum_{k=1}^\infty \frac{c' \alpha \log(1/\alpha) + c''\log\left(\frac{2k^2}{\delta}\right)}{\alpha 2^{-2k} } \max\{1,(\tfrac{5}{4}) (\tfrac{1}{5})^{k-k_*}\} \\
&\leq \sum_{k=1}^{k_*} \frac{c'''\log\left(\frac{2k_*^2}{\delta}\right)}{\alpha} 4^k  + 5^{k_*} \tfrac{c'''}{\alpha} \sum_{k=k_*+1}^\infty \left( 2\log(k) + \log(\tfrac{2}{\delta}) \right) (\tfrac{4}{5})^{k} \\
&\leq \frac{c'''\log\left(\frac{2k_*^2}{\delta}\right)}{\alpha} 4^{k_*+1}  + 5^{k_*} \tfrac{c'''}{\alpha} \sum_{k=k_*+1}^\infty \left( 2\log(k) + \log(\tfrac{2}{\delta}) \right) (\tfrac{4}{5})^{k} \leq \frac{c''''\log\left(\frac{k_*}{\delta}\right)}{\alpha} (2^{k_*})^2
\end{align*}
since $\sup_\alpha  \alpha \log(1/\alpha)  \leq e^{-1}$ and 
\begin{align*}
\sum_{k=k_*}^\infty \log(k) (\tfrac{4}{5})^{k} \hspace{-.5in}&\hspace{.5in}= \sum_{k=k_*}^{2k_*-1} \log(k) (\tfrac{4}{5})^{k} + \sum_{k=2k_*}^{\infty}  \log(k) (\tfrac{4}{5})^{k/2} (\tfrac{4}{5})^{k/2} \\
&\leq \log(2k_*) \sum_{k=k_*}^{2k_*-1} (\tfrac{4}{5})^{k} + \sum_{k=2k_*}^{\infty}  (\tfrac{4}{5})^{k/2} \leq \left( \log(2k_*)  + 2 \right) \sum_{k=k_*}^{\infty}  (\tfrac{4}{5})^{k} = 5 \log(2e^2k_*)  (\tfrac{4}{5})^{k_*}
\end{align*}
since $\sup_k  \log(k) (\tfrac{4}{5})^{k/2} \leq 1$. Noting that $k_* \leq \log_2(\tfrac{1}{\theta_1-\theta_0})+1$ completes the proof.
}

Now we consider the case when $\theta_1-\theta_0$ is known but a lower bound on $\alpha$ is not.
The following theorem characterizes the performance of Algorithm~\ref{alg:unknown_alpha}.
\begin{theorem}[Unknown $\alpha$, known $\theta_0,\theta_1$] \label{unknown_alpha}
Fix $\delta \in (0,1)$. If Algorithm~\ref{alg:unknown_alpha} is run with $\delta, \theta_1 - \theta_0$ then with probability at least $1-\delta$ a heavy distribution is returned and the the expected number of total samples taken is no more than
\begin{align*}
\frac{c  \log\left( \log\left(\tfrac{1}{\alpha}\right) / \delta \right)  }{\alpha(\theta_1-\theta_0)^2}
\end{align*}
for an absolute constant $c$. 
\end{theorem}
\makeproof{UnknownAlphaProof}{The proof of this result is nearly identical to that of Theorem~\ref{unknown_epsilon} except the following changes. Let $K$ be the random stage in which Algorithm~\ref{alg:unknown_alpha} outputs a distribution and let $k_*$ be the smallest $k\in\mathbb{N}$ that satisfies $2^{-k} \leq \alpha$. Moreover, if $M_k$ is the number of measurements taken at stage $k$, then by Wald's identity expected number of measurements is bounded by
\begin{align*}
\E\left[ \sum_{k=1}^K M_k \right] &= \sum_{k=1}^\infty \E[N_k] \P(K\geq k) \leq \sum_{k=1}^\infty \frac{c' \alpha \log(2^k) + c''\log\left(\frac{2k^2}{\delta}\right)}{2^{-k} \epsilon^2 } \max\{1,(\tfrac{5}{4}) (\tfrac{1}{5})^{k-k_*}\} \\
&\leq \sum_{k=1}^{k_*} \frac{c'''\log\left(\frac{2k_*^2}{\delta}\right)}{\epsilon^2} 2^k  + 5^{k_*} \tfrac{c'''}{\epsilon^2} \sum_{k=k_*+1}^\infty \left( \alpha k \log(2) + 2\log(k) + \log(\tfrac{2}{\delta}) \right) (\tfrac{2}{5})^{k} \\
&\leq \frac{c''''\left(\alpha k_* +  \log\left(\frac{k_*}{\delta}\right)\right)}{\epsilon^2} 2^{k_*} \leq \frac{ c''''' \log\left( \log(1/\alpha)/\delta \right)}{\alpha \epsilon^2}
\end{align*}
by the same series of steps as the proof of Theorem~\ref{unknown_epsilon} and the fact that $\sum_{k=n}^\infty k a^{k} \leq \frac{n a^{n}}{(1-a)^2}$ for any $a \in (0,1)$. The final inequality follows from $k_* \leq \log_2(1/\alpha) + 1$ and that $\alpha k_* = \alpha \log_2(2/\alpha) \leq 2$.
}

\subsection{Fully adaptive strategies when $\alpha,\theta_0,\theta_1$ are unknown}
\label{subsec:upper_adaptive_unknown}
We now consider the most difficult setting in which no prior knowledge about $\alpha,\theta_0,\theta_1$ are known. The algorithm for this setting, Algorithm~\ref{alg:upper-adaptive-unknown}, requires a more sophisticated argument than the simple ``doubling trick'' used above when partial information was available. As far as we are aware this is the first result of its kind that does not require any prior estimation or knowledge of the unknown mean distribution parameters. We also remark that the placing of ``landmarks'' ($\alpha_k,\epsilon_k$) throughout the search space as is done in Algorithm~\ref{alg:upper-adaptive-unknown} can also be generalized to generic infinite armed bandit problems, perhaps providing a simple alternative to the two-stage approach of estimation then exploration of \cite{carpentier2015simple}.

\begin{algorithm}
\begin{framed}
\textbf{Given} $\delta > 0$. \\
\textbf{Initialize} $\ell = 1$, heavy distribution $h$ = \texttt{null}. \\
\textbf{Repeat} until $h$ is not \texttt{null}: \\
    \forceindent \textbf{Set} $\gamma_\ell = 2^\ell, \delta_\ell = \delta / (2\ell^3)$  \\
    \forceindent \textbf{Repeat} for $k = 0,\dots,\ell$: \\
        \forceindent \forceindent \textbf{Set} $\alpha_k = \frac{2^k}{\gamma_\ell}, \epsilon_k = \sqrt{\frac{1}{2\alpha_k\gamma_\ell}}$ \\
        \forceindent \forceindent \textbf{Run} Algorithm~\ref{alg:upper-adaptive-bounded} with $\alpha_0 = \alpha_k, \epsilon_0 = \epsilon_k, \delta = \delta_\ell $ and \textbf{Set} $h$ to its output. \\
        \forceindent \forceindent \textbf{If} $h$ is not \texttt{null} \textbf{break} \\
    \forceindent \textbf{Set} $\ell = \ell + 1$ \\
\textbf{Output} $h$
\end{framed}
\caption{Adaptive strategy for heavy distribution identification with unknown parameters}
\label{alg:upper-adaptive-unknown}
\end{algorithm}

\begin{theorem}[Unknown $\alpha,\theta_0,\theta_1$]\label{unknown_all} 
Fix $\delta \in (0,1)$. If Algorithm~\ref{alg:upper-adaptive-unknown} is run with $\delta$ then with probability at least $1-\delta$ a heavy distribution is returned and the expected number of total samples taken is bounded by 
\begin{align*}
c  \frac{\log_2(\tfrac{1}{\alpha \epsilon^2})}{\alpha \epsilon^2} (\alpha \log_2(\tfrac{1}{ \epsilon^2}) +\log(\log_2(\tfrac{1}{\alpha \epsilon^2})) + \log(1/\delta)) 
\end{align*}
for an absolute constant $c$.
\end{theorem}

\makeproof{UnknownAllProof}{
The proof is broken up into a few steps, summarized as follows. For any given $\alpha_0,\epsilon_0$, Theorem~\ref{upper-adaptive-bounded} takes just $O\left(\frac{\alpha\log(1/\alpha_0) + \log(1/\delta)}{\alpha_0 \epsilon_0^2}\right)$ samples in expectation and the procedure makes an error (i.e. returns a light distribution) with probability less than $\delta$. Define $\epsilon = \theta_1-\theta_0$. In addition, if$\epsilon = \theta_1-\theta_0$, $\alpha \geq \alpha_0$, and $\epsilon \geq \epsilon_0$ then with probability at least $4/5$ a heavy distribution is returned after the same expected number of samples. We will leverage this result to show that if we are given an upper bound $\gamma_0$ such that $\frac{1}{\alpha \epsilon^2} \leq \gamma_0$ then it is possible to identify a heavy distribution with probability at least $4/5$ using just $O\left(  \log_2(\gamma_0) \gamma_0 \left[\alpha\log_2(\gamma_0) + \log( \log_2 (\gamma_0) /\delta)\right] \right)$ samples in expectation. Finally, we apply the ``doubling trick'' to $\gamma$ so that even though the tightest $\gamma$ is not known a priori, we can adapt to it using only twice the number of samples as if we had known it. Because each of the stages is independent of one another, the probability that the procedure requires more than $\ell_*+i$ stages is less than $(1/5)^i$, which yields our expected sample complexity.

For all $\ell \in \mathbb{N}$ define $\delta_{\ell} = \frac{\delta}{2\ell^3}$ and $\gamma_\ell = 2^\ell$. 
Fix some $\ell$ and consider the set $\{ (\alpha,\epsilon) : \frac{1}{\alpha \epsilon^2} = \gamma_\ell \}$. 
Clearly, in this set, $\alpha \in [1/\gamma_\ell,1/2]$. For all $k \in \{0,\dots,\ell-1\}$, define $\alpha_k = \frac{2^k}{\gamma_\ell}$ and $\epsilon_k = \sqrt{ \frac{1}{2 \alpha_k \gamma_\ell} }$. The key observation is that
\begin{align} \label{cover_claim}
 \{ (\alpha,\epsilon) : \frac{1}{\alpha \epsilon^2} \leq \gamma_\ell \} \subseteq \bigcup_{k=0}^{\ \log_2 \gamma_\ell  -1} \{ (\alpha,\epsilon) :  \alpha \geq \alpha_k, \epsilon \geq \epsilon_k \} .
\end{align}
To see this, fix any $(\alpha',\epsilon')$ such that $\frac{1}{\alpha' \epsilon'^2} \leq \gamma_\ell$. 
Let $k_*$ be the integer that satisfies $\alpha_{k_*} \leq \alpha' < 2\alpha_{k_*}$. Such a $k_*$ must exist since $\alpha_{\ell-1} = \frac{1}{2} \geq \alpha' \geq \frac{1}{\gamma_\ell \epsilon'^2} \geq \frac{1}{\gamma_\ell} = \alpha_0$. Then $\gamma_\ell \geq \frac{1}{\alpha' \epsilon'^2} \geq \frac{1}{ 2 \alpha_{k_*} \epsilon'^2}$ which means $\epsilon' \geq \sqrt{\frac{1}{2 \alpha_{k_*}\gamma_\ell}} = \epsilon_{k_*}$ which proves the claim of \eqref{cover_claim}. 
Consequently, even if no information about $\alpha$ or $\epsilon$ individually is known but $\frac{1}{\alpha \epsilon^2} \leq \gamma_\ell$, one can cover the entire range of valid $(\alpha,\epsilon)$ with just $\log_2(\gamma_\ell) = \ell$ landmarks $(\alpha_k,\epsilon_k)$. 

For any $\ell \in \mathbb{N}$ and $k \in \{0,\dots,\ell-1\}$, if Algorithm~\ref{alg:upper-adaptive-bounded} is used with $\alpha_0 = \alpha_k, \epsilon_0 = \epsilon_k$ and $\delta = \delta_\ell$ then the probability that a light distribution is returned, declared heavy is less than $\delta_{\ell}$. And the probability that a light distribution is returned, declared heavy for {\em any} $\ell \in \mathbb{N}$ and $k \in \{0,\dots,\ell-1\}$ is less than $\sum_{\ell=1}^\infty \ell \delta_\ell = \delta \sum_{\ell=1}^\infty \ell/(2\ell^3) \leq \delta$. Thus, given that Algorithm~\ref{alg:upper-adaptive-unknown} terminates with a non-null distribution $h$, $h$ is heavy with probability at least $1-\delta$. This proves correctness. We next bound the expected number of samples taken before the procedure terminates. 

With the inputs given in the last paragraph for any $k,\ell$, Algorithm~\ref{alg:upper-adaptive-bounded} takes an expected number samples bounded by $c \gamma_\ell(\alpha\log(1/\alpha_k) +\log(1/\delta_\ell))$. Let $L \in\mathbb{N}$ be the random stage at which Algorithm~\ref{alg:upper-adaptive-unknown} terminates with a non-null distribution $h$. Let $\ell_*$ be the first integer such that there exists a $k \in \{0,\dots,\ell_*-1\}$ with $\alpha \geq \alpha_k$ and $\epsilon \geq \epsilon_k$ (recall that in this case $\frac{1}{\alpha_k\epsilon_k^2} \leq \gamma_{\ell_*}$). Then by the end of stage $\ell \geq \ell_*$, at most $c \ell \gamma_{\ell}(\alpha\log(\gamma_\ell) +\log(1/\delta_{\ell}))$ samples in expectation were taken on stage $\ell$ and with probability at least $4/5$ the procedure terminated with a heavy coin. By the independence of samples between rounds, observe that $\P(L \geq \ell_* + i) = \sum_{j=i}^\infty \P(L = \ell_* + j) \leq (\tfrac{5}{4}) (\tfrac{1}{5})^i$. Thus, if $M_\ell$ is the number of samples taken at stage $\ell$ then by Wald's identify, the total expected number of samples taken before termination is bounded by
\begin{align*}
\E\left[ \sum_{\ell=1}^{L} c \ell \gamma_{\ell}(\alpha\log(\gamma_\ell)+\log(1/\delta_{\ell})) \right] \hspace{-2in}&\hspace{2in}= \sum_{\ell=1}^{\infty} \E[M_\ell] \P(L\geq\ell) \leq \sum_{\ell=1}^{\infty} c \ell \gamma_{\ell}(\alpha\log(\gamma_\ell)+\log(1/\delta_{\ell})) \P(L\geq\ell) \\
&\leq \sum_{\ell=1}^{\ell_*} c \ell \gamma_{\ell}(\alpha\log(\gamma_\ell) +\log(1/\delta_{\ell})) + \sum_{\ell=\ell_*+1}^\infty c \ell \gamma_{\ell}(\alpha\log(\gamma_\ell) +\log(1/\delta_{\ell})) (\tfrac{5}{4}) (\tfrac{1}{5})^{\ell-\ell_*} \\
&\leq \sum_{\ell=1}^{\ell_*} c \ell 2^{\ell}(\alpha \ell +\log(2\ell^3/\delta)) + \sum_{\ell=\ell_*+1}^\infty c \ell 2^{\ell}(\alpha\ell +\log(2\ell^3/\delta)) (\tfrac{5}{4}) (\tfrac{1}{5})^{\ell-\ell_*} \\
&\leq c \ell_*  (\alpha \ell_* +\log(2\ell_*^3/\delta)) \sum_{\ell=1}^{\ell_*} 2^{\ell} + c (\tfrac{5}{4}) 5^{\ell_*} \sum_{\ell=\ell_*+1}^\infty  \left( \alpha \ell^2 (\tfrac{2}{5})^{\ell} + 3 \ell \log(\ell) (\tfrac{2}{5})^{\ell} + \log(2/\delta) \ell (\tfrac{2}{5})^{\ell}  \right) \\
&\leq 2 c \ell_* 2^{\ell_*} (\alpha \ell_* +\log(2\ell_*^3/\delta)) \\
&\hspace{.25in}+ c (\tfrac{5}{4})  5^{\ell_*}  \left( 2\alpha (\ell_*+1)^2 (\tfrac{2}{5})^{\ell_*} + 12 \log(2e^2 \ell_*) (\ell_*+1) (\tfrac{2}{5})^{\ell_*} + 4 \log(2/\delta) (\ell_*+1) (\tfrac{2}{5})^{\ell_*}  \right)\\
&\leq c' \ell_* 2^{\ell_*} (\alpha \ell_* +\log(\ell_*) + \log(1/\delta))  
\end{align*}
for some absolute constant $c'$ since $\sum_{k=n}^\infty k a^k \leq \frac{n a^n}{(1-a)^2}$, $\sum_{k=n}^\infty k^2 a^k \leq \frac{n^2 a^n}{(1-a)^3}$, and 
\begin{align*}
\sum_{\ell=\ell_*+1}^\infty \ell \log(\ell) (\tfrac{2}{5})^{\ell} &\leq \log(2\ell_*) \sum_{\ell_*+1}^{2\ell_*} \ell (\tfrac{2}{5})^{\ell} + \sum_{2\ell_*+1}^\infty \ell (\tfrac{2}{5})^{\ell/2} \left( \log(\ell) (\tfrac{2}{5})^{\ell/2}\right)  \\
&\leq  \log(2 e^2 \ell_*) \sum_{\ell_*+1}^{ \infty} \ell (\tfrac{2}{5})^{\ell} \leq 4  \log(2 e^2 \ell_*) (\ell_*+1) (\tfrac{2}{5})^{\ell_*}
\end{align*}
since  $\max_{x\geq1} \log(x) (\tfrac{2}{5})^{x/2} \leq 1$.
Noting that $\ell_* \leq \log_2(\frac{1}{\alpha \epsilon^2})+1$, we have that the total number of samples, in expectation, is bounded by
\begin{align*}
c' \ell_* 2^{\ell_*} (\alpha \ell_* +\log(\ell_*) + \log(1/\delta)) &\leq c''  \frac{\log_2(\tfrac{1}{\alpha \epsilon^2})}{\alpha \epsilon^2} (\alpha \log_2(\tfrac{1}{\alpha \epsilon^2}) +\log(\log_2(\tfrac{1}{\alpha \epsilon^2})) + \log(1/\delta)) \\
&\leq c'''  \frac{\log_2(\tfrac{1}{\alpha \epsilon^2})}{\alpha \epsilon^2} (\alpha \log_2(\tfrac{1}{ \epsilon^2}) +\log(\log_2(\tfrac{1}{\alpha \epsilon^2})) + \log(1/\delta)) 
\end{align*}
where we've used the fact that $\sup_{\alpha \in [0,1]}\alpha \log(1/\alpha) \leq e^{-1}$.
}

\section{Conclusion}
In this work, we prove upper and lower bounds on the complexity of detecting mixture distributions with partial or missing knowledge of the distribution parameters.
We note that there is still a $\log$-factor gap between several of our upper and lower bounds, and investigating whether either can be tightened remains an interesting problem.
Importantly, in this work we considered mixtures of only two components, whereas the literature on infinite-armed bandits considers a continuous mixture.
Extending the algorithms developed for our upper bounds to the continuous mixture case is a promising direction, as it would represent the first such algorithm that does not rely on knowledge of the distribution parameters or estimating them first with a two-stage approach.

\vspace{0.4cm}
\noindent\textbf{Acknowledgments}

\vspace{0.1cm}
\noindent Kevin Jamieson is generously supported by ONR awards N00014-15-1-2620, and  N00014-13-1-0129. This research is supported in part by NSF CISE Expeditions Award CCF-1139158, DOE Award SN10040 DE-SC0012463, and DARPA XData Award FA8750-12-2-0331, and gifts from Amazon Web Services, Google, IBM, SAP, The Thomas and Stacey Siebel Foundation, Apple Inc., Arimo, Blue Goji, Bosch, Cisco, Cray, Cloudera, Ericsson, Facebook, Fujitsu, Guavus, HP, Huawei, Intel, Microsoft, Pivotal, Samsung, Schlumberger, Splunk, State Farm and VMware.

\newpage

\bibliography{mostBiasedCoin.bbl}

\newpage
\appendix

\appendixproofsection{Proofs of Lower Bounds}
\appendixproof[Claim]{correctness_claim}{CorrectnessClaimProof}
\appendixproof{fixed_known}{FixedKnownProof}
\appendixproof[Corollary]{known_all_lower}{KnownAllLowerProof}
\appendixproof{exp_family}{ExpFamilyProof}
\appendixproof[Corollary]{fixed_unknown}{FixedUnknownProof}

\appendixproofsection{Proofs of Upper Bounds}
\appendixproof{fixed_upper_known}{FixedUpperKnownProof}
\appendixproof*[Theorem]{upper-adaptive-bounded}{KnownAllProof}
\appendixproof{unknown_epsilon}{UnknownEpsilonProof}
\appendixproof{unknown_alpha}{UnknownAlphaProof}
\appendixproof{unknown_all}{UnknownAllProof}

\section{Gaussians}\label{Gaussian_discussion}
\subsection{On the detection of a mixture of Gaussians} \label{gaussian_bounds}

For known $\sigma^2$, consider the hypothesis test of Problem~\ref{hyp_test_normal}
. In what follows, let $\chi^2(  \theta_1,\theta_0  )$ and  $KL(  \theta_1,\theta_0  )$ be the chi-squared and KL divergences of the two distributions of $\H_1$. Note that for $\frac{(\theta_1 - \theta_0)^2}{\sigma} \leq 1$, we have that $\chi^2(  \theta_1,\theta_0  ) = e^{\frac{(\theta_1-\theta_0)^2}{\sigma^2}}-1 \leq 2 \frac{(\theta_1-\theta_0)^2}{\sigma^2} = 4 KL(\theta_1,\theta_0)$

Theorem~\ref{fixed_known} says that for $\frac{(\theta_1 - \theta_0)^2}{\sigma^2} \leq 1$, a procedure that has maximum probability of error less than $\delta$ requires at least $\max\left\{ \frac{1-\delta}{\alpha}, \frac{\log(1/\delta)}{4 \alpha^2 KL( \theta_1,\theta_0 ) } \right\}$ samples to decide the above hypohesis test, even if $\alpha,\theta_0,\theta_1$ are known. The next subsection shows that if $\alpha,\theta_0,\theta_1$ are unknown then one requires at least $\frac{\log(1/\delta)}{2 [\alpha KL( \theta_1,\theta_0 )]^2 }$ samples to decide the above hypothesis test correctly with probability at least $1-\delta$. This is likely achievable using the method of moments \citep{hardt2014sharp}.
\subsection{Lower bounds}

\begin{theorem}\label{gaussian_lower_bound}
For known $\sigma^2$, consider the hypothesis test of Problem~\ref{hyp_test_normal}. If $\theta_* = (1-\alpha) \theta_0 + \alpha \theta_1$ and $\frac{\theta_1-\theta_0}{\sigma} \leq 1$ then
\begin{align*}
\chi^2( (1-\alpha) f_{\theta_0}(x) + \alpha f_{\theta_1}(x) | f_{\theta_*}(x) )  &\leq c' \left( \alpha(1-\alpha) \frac{(\theta_1-\theta_0)^2}{\sigma^2} \right)^2 
\end{align*}
for some absolute constant $c'$.
\end{theorem}
\begin{proof}
If $f_\theta = \mathcal{N}(\theta,\sigma^2)$ then $f_\theta(x) = h(x) \exp( \eta(\theta) x - b(\theta) )$ where $h(x) = \tfrac{1}{\sqrt{2\pi\sigma^2}} e^{-\frac{x^2}{2\sigma^2}}$, $\eta(\theta) = \frac{\theta}{\sigma^2}$, and $b(\eta(\theta)) =  \frac{\eta(\theta)^2 \sigma^2}{2} = \frac{\theta^2}{2\sigma^2}$. Thus,
\begin{align*}
\theta_* = \eta^{-1}\big( (1-\alpha) \eta(\theta_0) + \alpha \eta(\theta_1) \big) = (1-\alpha) \theta_0 + \alpha \theta_1
\end{align*}
and
\begin{align*}
\sup_{y \in [\theta_0,\theta_1]}   b( 2\eta(y)-\eta(\theta_*)) - (2b(\eta(y))-b(\eta(\theta_*)))  = \sup_{y \in [\theta_0,\theta_1]} \frac{(y-\theta_*)^2}{\sigma^2} \leq  \frac{(\theta_1-\theta_0)^2}{\sigma^2} =: \kappa
\end{align*}
and 
\begin{align*}
\sup_{x \in [ \dot{b}(\eta(\theta_-)),\dot{b}(\eta(\theta_+))]} f_{\dot{b}^{-1}(x)}( x ) = \sup_{x \in [\theta_-,\theta_+]} \sup_{\theta \in \R} \frac{1}{\sqrt{2 \pi \sigma^2}}e^{-\frac{(x-\theta)^2}{2\sigma^2}} \leq \frac{1}{\sqrt{2 \pi \sigma^2}} =: \gamma.
\end{align*}
Note that for any $ \theta < \theta'$ we have $\dot{b}( \eta(\theta') ) - \dot{b}( \eta(\theta)) = \theta'-\theta$, $M_2(\theta) = \sigma^2$, and $M_4(\theta) = 3 \sigma^4$. Plugging these values into the theorem we have
\begin{align*}
\hspace{.5in}&\hspace{-.5in}c= e^\kappa \bigg(\sup_{\theta \in [\theta_0,\theta_1]} M_2(\theta)^2 \ \left( 2 + \gamma \left(\dot{b}(\eta(\theta_+))-\dot{b}(\eta(\theta_-))\right)  \right) \\
 &+ 8 M_4(\theta_- ) + 8 M_4( \theta_+ )  + 16\left( \dot{b}( \eta(\theta_+) ) - \dot{b}( \eta(\theta_-)) \right)^4  + \tfrac{2}{5} \gamma \left( \dot{b}(\eta(\theta_+))-\dot{b}(\eta(\theta_-))\right)^5 \bigg) \\
 =& e^{\frac{(\theta_1-\theta_0)^2}{\sigma^2} } \bigg( \sigma^4 \ \left( 2 +  \frac{2(\theta_1-\theta_0)}{\sqrt{2\pi} \sigma}  \right) + 48 \sigma^4+ 256\left(\theta_1-\theta_0\right)^4  + \tfrac{64}{5 \sqrt{2\pi}} \ \frac{\left(\theta_1-\theta_0\right)^5}{\sigma} \bigg) 
\end{align*}
noting that $\theta_+-\theta_- = 2(\theta_1-\theta_0)$. If $\frac{\theta_1-\theta_0}{\sigma} \leq 1$ then $c = c' \sigma^4$ for some absolute constant $c'$ and $\left( \eta(\theta_1)-\eta(\theta_0)\right)^2 = \frac{(\theta_1-\theta_0)^2}{\sigma^4}$ which yields the final result.
\end{proof}

\subsection{Gaussian Upper bound for known $\alpha,\theta_0,\theta_1$} \label{gaussian_upper_known}
For known $\sigma^2$, consider the hypothesis test of Problem~\ref{hyp_test_normal} with $\theta=\theta_0$. We observe a sample $X_1,\dots,X_n$ and are trying to establish whether it came from $\mathbf{H}_0$ or $\mathbf{H}_1$. 

Consider the test 
\begin{align*}
 \frac{1}{n} \sum_{i=1}^n \1_{X_i > \theta_1 } \mathop{\gtrless}_{\mathbf{H}_0}^{\mathbf{H}_1}  \frac{ \P_1( X_1 >  \theta_1 ) + \P_0( X_1 >  \theta_1 ) }{2} =: \gamma.
\end{align*}

If $\epsilon = \P_1( X_1 >  \theta_1 ) -   \P_0( X_1 >  \theta_1 )$ then 
\begin{align*}
\P_1 \left( \frac{1}{n} \sum_{i=1}^n \1_{X_i > \theta_1 }  \leq \gamma \right) &= \P_1 \left( \frac{1}{n} \sum_{i=1}^n \1_{X_i > \theta_1 }  \leq  \P_1(X_1 > \theta_1) - \epsilon/2 \right) \leq e^{-n \epsilon^2/2}
\end{align*}
and
\begin{align*}
\P_0\left( \frac{1}{n} \sum_{i=1}^n \1_{X_i > \theta_1 } \geq \gamma\right) = \P_0 \left( \frac{1}{n} \sum_{i=1}^n \1_{X_i > \theta_1 }  \geq  \P_0(X_1 > \theta_1) +\epsilon/2 \right)  \leq e^{-n \epsilon^2/2}
\end{align*}
by sub-Gaussian tail bounds.
If $Q(x) = \int_x^\infty \frac{1}{\sqrt{2\pi}} e^{-z^2/2} dz$ and $\Delta =  \frac{\theta_1 - \theta_0}{\sigma} $ then 
\begin{align*}
\P_0( X_1 >  \theta_1 ) &= Q\left( \Delta \right) \\
\P_1( X_1 >  \theta_1 ) &= (1-\alpha) Q\left( \Delta \right) + \alpha \frac{1}{2}
\end{align*}
so 
\begin{align*}
\epsilon = \alpha\left( \frac{1}{2} - Q\left( \Delta \right)  \right) = \alpha \int_0^{\Delta} \frac{1}{\sqrt{2\pi}} e^{-x^2/2} dx \geq \min\{  \frac{\alpha \Delta}{ 4 \sqrt{2 \pi} } , \frac{1}{4}\alpha\}.
\end{align*}
Thus, the test fails with probability at most
\begin{align*}
\exp\left[ -n \alpha^2 \min\left\{ \frac{(\theta_1-\theta_0)^2}{64 \pi \sigma^2}, \frac{1}{32} \right\} \right].
\end{align*}
We conclude that if $\Delta =  \frac{\theta_1 - \theta_0}{\sigma} \leq 1$ and $n \geq \frac{(\theta_1-\theta_0)^2 \log(1/\delta)}{64 \pi \alpha^2\sigma^2} =  \frac{KL(\P_{\theta_1},\P_{\theta_0}) \log(1/\delta)}{64 \pi \alpha^2} $ the correct hypothesis is selected. The $1/\alpha$ sufficiency result holds for large enough $\Delta$ since one merely needs to observe just one sample since the probability of it coming from $\theta_0$ is negligible.

\end{document}